\documentclass[twoside]{article}

\usepackage[accepted]{aistats2023}
\usepackage[round]{natbib}

\usepackage{bibentry}

\usepackage{tikz}
\usepackage{booktabs} 

\usepackage{caption}
\usepackage{subcaption}

\usepackage{amsmath}
\usepackage{amssymb}
\usepackage{mathtools}
\usepackage{amsthm}
\usepackage{bm}
\usepackage{multirow}
\usepackage{multicol}

\usepackage[capitalize,noabbrev]{cleveref}
\Crefname{section}{Section}{Sections}
\crefname{section}{Sec.}{Secs.}
\Crefname{table}{Table}{Tables}
\crefname{table}{Tab.}{Tabs.}
\Crefname{appsec}{Appendix}{Appendices}
\crefname{appsec}{App.}{Apps.}

\usepackage{breakurl}
\usepackage{url}

\theoremstyle{plain}
\newtheorem{theorem}{Theorem}
\newtheorem{proposition}{Proposition}
\newtheorem{lemma}{Lemma}
\newtheorem{corollary}{Corollary}
\theoremstyle{definition}

\theoremstyle{remark}
\newtheorem{remark}{Remark}

\newcommand{\overbar}[1]{\mkern 4mu\overline{\mkern-4mu#1\mkern-1mu}\mkern 1mu}

\def\R{\mathbb{R}}

\usepackage{enumitem,amssymb}
\newlist{todolist}{itemize}{2}
\setlist[todolist]{label=$\square$}
\usepackage{pifont}

\def\R{\mathbb{R}}

\def\E{\mathbb{E}}

\def\1{\mathds{1}}

\newcommand{\mm}[1]{#1}

\usepackage[colorinlistoftodos]{todonotes}

\usepackage{xcolor}

\begin{document}

%

%

\twocolumn[

\aistatstitle{Bounding Evidence and Estimating Log-Likelihood in VAE}

\aistatsauthor{ \L{}ukasz Struski \And Marcin Mazur \And  Paweł Batorski \And Przemys\l{}aw Spurek \And Jacek Tabor}

\aistatsaddress{
\\[-8pt]
Jagiellonian University, Faculty  of Mathematics and Computer Science, Kraków, Poland\\
\{lukasz.struski, marcin.mazur, przemyslaw.spurek, jacek.tabor\}@uj.edu.pl
} ]

\begin{abstract}
Many crucial problems in deep learning and statistical inference
are caused by a variational gap, i.e., a difference between model evidence (log-likelihood) and evidence lower bound (ELBO). 
In particular, in a classical VAE setting that involves training via an ELBO cost function, it is difficult to provide a robust comparison of the effects of training between models, since we do not know a log-likelihood of data (but only its lower bound). In this paper, to deal with this problem, we introduce a general and effective upper bound, which allows us to efficiently approximate the evidence of data. We provide extensive theoretical and experimental studies of our approach, including its comparison to the other state-of-the-art upper bounds, as well as its application as a tool for the evaluation of models that were trained on various lower bounds.
\end{abstract}

\section{INTRODUCTION}
Many important models in deep learning 
\citep{bayer2021mind,burda2015importance,kingma2013auto}, reinforcement learning \citep{duo2021improving,todorov2008general,toussaint2006probabilistic} and statistical inference \citep{gao2017bounds,khan2020new} suffer from the existence of a variational gap\footnote{In fact, in such a general context it is rather known as Jensen's gap, but here and henceforth we call it consequently the variational gap.}, which means a difference between the evidence and its lower bound (which follows from Jensen's inequality), i.e.:
\begin{equation}\label{eq:variationalgap} 
\text{variational gap} =  f(\E X)-\E f(X),
\end{equation}
where $X$ is a random variable and $f$ is a concave function. 
A simple visualization of this effect can be delivered using the concave function $f(x)=-x^2$ to transform Gaussian random variable $X \sim \mathcal{N}(0,1)$. Indeed,  in this case $ f(\E X) =0 > -1 =\E f(X)$ (note that $-f(X) \sim \chi^2(1)$).

\begin{table}[!ht]\small
\centering
\caption{Estimated size of various variational gap bounds for the evidence of data (lower is better), calculated for VAE, IWAE-5, and IWAE-10 models, previously trained on MNIST, SVHN, and CelebA datasets. All computations were averaged over 3 collections of $2^{16}$ latent samples, and over the test dataset.}
\label{tab.gap}
\begin{tabular}{@{}c@{\!}l@{\;\,}r@{\;\;}r@{\;\;}r@{\;\;}r@{\;\;}r@{\;\;}r@{\;\;}r@{\;\;}r@{\;\;}r@{}}
\toprule
& & \multicolumn{8}{c}{\bfseries VARIATIONAL GAP BOUND (VG-B)} \\
\cmidrule{3-10}
  \rotatebox{77}{\bfseries DATASET} & \rotatebox{77}{\bfseries MODEL} & \rotatebox{77}{\small \bfseries IS (OUR)} & \rotatebox{77}{\small \bfseries $\text{CUBO}_{1.5}$} & \rotatebox{77}{\small \bfseries $\text{CUBO}_2$} & \rotatebox{77}{\small \bfseries EUBO} & \rotatebox{77}{\small \bfseries $\text{TVO}_2$} & \rotatebox{77}{\small \bfseries $\text{TVO}_5$} & \rotatebox{77}{\small \bfseries $\text{TVO}_{10}$} & \rotatebox{77}{\small \bfseries $\text{TVO}_{50}$} \\
\midrule
 
\multirow[c]{3}{*}{\rotatebox{90}{MNIST}} & VAE & 0.04 & 1.59 & 2.83 & 6.68 & 3.34 & 1.34 & 0.67 & 0.13 \\
 & IWAE-5 & 0.004 & 1.32 & 2.24 & 12.38 & 6.19 & 2.48 & 1.24 & 0.25 \\
 & IWAE-10 & 0.01 & 1.54 & 2.55 & 18.03 & 9.01 & 3.61 & 1.80 & 0.36 \\
\cmidrule{1-10}
\multirow[c]{3}{*}{\rotatebox{90}{SVHN}} & VAE & 0.44 & 3.16 & 4.84 & 21.08 & 10.54 & 4.22 & 2.11 & 0.42 \\
 & IWAE-5 & 0.31 & 2.74 & 4.30 & 28.82 & 14.41 & 5.76 & 2.88 & 0.58 \\
 & IWAE-10 & 0.28 & 2.83 & 4.40 & 32.00 & 16.00 & 6.40 & 3.20 & 0.64 \\
\cmidrule{1-10}
\multirow[c]{3}{*}{\rotatebox{90}{CelebA}} & VAE & 1.12 & 3.25 & 4.92 & 45.45 & 22.73 & 9.09 & 4.55 & 0.91 \\
 & IWAE-5 & 1.23 & 3.40 & 5.14 & 85.41 & 42.70 & 17.08 & 8.54 & 1.71 \\
 & IWAE-10 & 0.72 & 3.22 & 4.86 & 75.46 & 37.73 & 15.09 & 7.55 & 1.51 \\
\bottomrule
\end{tabular}
\end{table}

In machine learning literature, where typically $f=\log$, various approximations of the true evidence $f(\E X)$ were proposed (see \Cref{sec:relwork} and references therein), but are often difficult to efficiently use in a deep neural network architecture. One of the reasons for this is that we train such models on mini-batches, and therefore the standard assumption is that the cost function factorizes as the sum over the input data set. In the other words, deep networks are designed to maximize (over network parameters) $\E f(X)$, rather than $f(\E X)$, which admits respective sample mean (unbiased) estimator.

However, there often naturally appear situations where maximization of $f(\E X)$ is an actual goal.
Probably the most important such case is the variational autoencoder (VAE)~\citep{kingma2013auto,pmlr-v32-rezende14}, which is 
one of the most popular autoencoder-based generative models.
Precisely, VAE uses an encoder network $q(z|x)$, which reduces the dimension of data and produces their latent codes (forced to follow approximately a given latent prior distribution $p(z)$), and a decoder network $p(x|z)$ that transforms the latent codes back to the data space. Both networks are jointly trained to maximize a variational lower bound for the log-likelihood (the evidence) of data: 
\begin{equation}\label{eq:evidence}
\begin{array}{l @{\;} l @{\;} l}
\text{evidence} & = & \log \E_{z\sim q(\cdot |x)} \frac{ p (x|z)p(z) }{q(z|x)},
\end{array}
\end{equation}
 which is known as evidence lower bound or (briefly) ELBO:
\begin{equation}\label{eq:elbo}
\begin{array}{l @{\;} l @{\;} l}
\text{ELBO} & = & \E_{z\sim q(\cdot |x)} \log \frac{ p (x|z) p(z) }{q(z|x)}.
\end{array}
\end{equation}

The use of ELBO, instead of a direct value of log-likelihood, seems to be a fundamental problem in VAE. In practice, such optimization can lead to learning suboptimal parameters \citep{burda2015importance}, when we mean that our final goal is an approximation of data distribution. Hence, estimating, bounding, and reducing a difference between the evidence and ELBO, i.e., the variational gap, became important issues investigated so far by many authors from the machine learning community (see \Cref{sec:relwork} for a respective overview of the literature). Moreover, such problems are (in a general context) strictly related to those concerning (reversed) Jensen's inequality, which is a subject of studies in various fields of pure and applied mathematics, including statistical inference \citep{brnetic2015refinement, horvath2021extensions,saeed2022refinements,dragomir2013some,jebara2001reversing,nielsen2010family}, reinforcement learning \citep{dayan1997using,williams2017information}, or even biological studies \citep{ruel1999jensen}.

Our work provides a comprehensive theoretical approach regarding the variational gap.
Most of all, we construct novel upper bounds for $f(\E X)$ (and hence also for the size of variational gap), which are given as expected values of some random variables that depend on $X$, and then, inspired by \citep{burda2015importance}, combine them with the technique of importance sampling, to derive tight bounds for the exact value of $f(\E X)$. Additionally, as an application in the field of deep learning, we use these general results for precisely estimating the log-likelihood of data for generative models, which are designed to learn only some lower bounds. Consequently, we obtain a method that allows comparing the effectiveness of the training process, which we examine using a few different experimental settings involving VAE-like architectures, i.e., the classical (Gaussian) VAE and two variants of the importance-weighted (Gaussian) autoencoder (IWAE)~\citep{burda2015importance}, all trained on MNIST, SVHN, and CelebA datasets.

Our contribution can be summarized as follows:
\vspace{-1.5\topsep}
\begin{itemize}[leftmargin=0.5cm]
    \item we introduce novel upper bounds for the variational gap, based on the importance sampling technique, which allows us to calculate a tight approximation of $f(\E X)$ for any concave function $f$, and provide their formal (mathematical) justification,
    \item we apply these results for $f=\log$, to provide precise estimates for the evidence (log-likelihood) of data, which we treat as a practical method for validating the effects of training in generative models involving lower bound optimization, and remark both benefits and limitations of our approach,
    \item we perform experiments that confirm our (theoretical) claims and prove the superiority of the proposed bound estimations for the variational gap (see \Cref{tab.gap}), in comparison to the other state-of-the-art techniques (we recall them briefly in \Cref{sec:relwork}, see also \Cref{app:bounds}).
\end{itemize}

\section{RELATED WORK}\label{sec:relwork}

One of the most popular generative, autoencoder-based models is variational autoencoder (VAE)~\citep{kingma2013auto,pmlr-v32-rezende14}, which aims to maximize the log-likelihood of data. However, since this likelihood is intractable, the main idea which stays behind estimating it is to calculate and optimize evidence lower bound (ELBO) instead, which results in appearing the variational gap.

One of the problems with the variational gap is its behavior since it can be tiny or tremendous, depending on the model distribution. The importance of taking care of gaps and their possible offending effects were mentioned in ~\citep{bayer2021mind}. There are several techniques to deal with the variational gap, such as its direct estimation~\citep{abramovich2016some} or finding an upper or lower bound, to know how much we can lose. For example, \citet{nowozin2018debiasing} and \citet{maddison2017filtering} create lower bounds using big-O notation. Bounds for the variational gap were also derived in~\citep{khan2020new} and later used for deriving new inequalities (e.g., bounding Csiszár divergence or converse of the Hölder inequality), as well as in \citep{grosse2015sandwiching}, where bidirectional Monte Carlo simulations were involved.

To our best knowledge, approaches the most related to the results of the present paper were introduced in~\citep{dieng2017variational}, where the authors proposed  the $\chi$ upper bound (CUBO) for the evidence of the data, which was expressed in terms of the $\chi$-divergence, and in~\citep{ji2019stochastic}, where the evidence upper bound (EUBO), involving the Kulback-Leibler divergence, was defined and explored. Following \citet{masrani2019thermodynamic}, it should be noticed that both EUBO and its generalization, i.e., an upper bound variant of the thermodynamic variational objective ($\text{TVO}$), can be derived as the right Riemann sum approximation of the log-likelihood of data expressed via thermodynamic integration method. Thus, we henceforth concentrate on the mentioned bounds when validating our approach, and (for readers' convenience) in \Cref{app:bounds} we provide basic information concerning them.  Nevertheless, we would like to emphasize that, unlike the others,  our approach possesses a more general theoretical background that does not exclude it from other potential applications\footnote{We treat this as motivation for our future work.}.

On the other hand, we can find a broad usage of variational inference not only in the context of generative models (to which we limit our considerations). For example, \citet{toussaint2006probabilistic} use it to present the expectation-maximization algorithm (EM) for computing optimal policies by solving Markov decision processes. Furthermore, \citet{botvinick2012planning} say that people use probabilistic inference when they plan, 
\citet{levine2018reinforcement} uses variational inference to derive a new view of reinforcement learning, where decision-making is an inference problem represented in a type of graphical model, and \citet{duo2021improving} proposes a policy optimization algorithm in the context of variational inference.

\section{THEORETICAL STUDY}

In this section, we present the main theoretical results of the paper. In the first subsection, in \Cref{th:1} we derive a general condition which, under the assumption of concavity, enables us to find an upper bound for the evidence. In the second subsection, inspired by~\citep{burda2015importance}, we describe the importance sampling technique that allows us to decrease the size of the gap by replacing the random variable with the mean of its independent copies.
The third subsection contains the crucial results of the paper (\Cref{th:3.7,th:euboimpc}), in which we provide a large class of upper bounds for the evidence. By respectively, choosing a suitable parameter $C$ in~\Cref{th:euboimpc}, this allows us to obtain for the map $f=\log$ much tighter estimations than those given directly by~\Cref{cor:burdalog} (see \Cref{cor:burdalogimproved} and the experimental results supplied in~\Cref{sec:exp}). In the last subsection, we try to answer the question of where the evidence lies in the interval given by the lower bound and the upper bound. In particular, we show in~\Cref{th:lognormest} that for an important class of log-normal distributions, it is located in the middle of that interval.

\subsection{Variational Gap}

Let $X$ be a random variable. By the classical Jensen inequality, for every concave function $f$ we have $f(\E X)\geq\E f(X)$.
The aim of this subsection is to obtain (under the above general assumptions) an upper bound for $f(\E X)$, which requires computing only the expected value of some random variable that depends on $X$. (Note that such an additional supposition ensures the additivity of the bound when applied for solving optimization problems by machine learning algorithms.) This easily follows from the following theorem, which can also be found in \citep{dragomir2013some} in a more general form. Nevertheless, for completeness, we include the proof in \Cref{app:proofs}.

\begin{theorem} \label{th:1}
Let $f$ be a smooth concave function. Then
\begin{equation}\label{eq:th:1}
\begin{array}{l @{\;} l @{\;} l}
f(\E X) & \leq & \E (f(X)+(Y-X)f'(X)),
\end{array}
\end{equation}
where $X$ and $Y$ are two independent random variables with the same distribution.
\end{theorem}

\Cref{th:1} and Jensen's inequality 
imply that the value of $f(\E X)$ is enclosed in the interval
\begin{equation}\label{eq:vargapint}
\begin{array}{l @{\;} l @{\;} l}
[\E f(X),\, \E f(X)+\E((Y-X) f'(X))].
\end{array}
\end{equation}
Thus, the size of the variational gap is bounded from above by the length of the interval \eqref{eq:vargapint}, i.e., by the value of \mbox{$\E[(Y-X) f'(X)]$}. 

Proceeding to the most important case of $f=\log$, we directly obtain the following corollary.

\begin{corollary}\label{cor:firstlog}
We have
\begin{equation}\label{eq:cor:1}
\begin{array}{l @{\;} l @{\;} l}
\E \log X &\leq & \log \E X \leq \E \log X+ \E\frac{Y}{X}-1,
\end{array}
\end{equation}
where $X$ and $Y$ are independent random variables with the same distribution.
\end{corollary}

\begin{remark}\label{rem:gamma}
Consider, for example\footnote{Here and henceforth, $\text{Gamma}(a,\theta)$ and $\text{Inv-Gamma}(a,\theta)$ denote Gamma and Inverse-Gamma distributions with the shape parameter $a>0$ and the scale parameter $\theta>0$, respectively.}, the random variable $X\sim \text{Gamma}(a,\theta)$ and the concave function $f=\log$. Then  $1/X\sim \text{Inv-Gamma}(a,1/\theta)$. Therefore, assuming $a>1$, we can calculate that
\begin{equation}\label{eq:betasample}
\begin{array}{l @{\;} l @{\;} l}
\E \frac{Y}{X} - 1 &  = & \E Y\, \E  \frac{1}{X} -1
 = \frac{a\theta}{\theta(a-1)}-1=\frac{1}{a-1},
\end{array}
\end{equation} 
which means that in this case, we cannot bound the variational gap effectively when $a$ approaches $1$. Taking into account the properties of the Gamma distribution, this phenomenon means that when we sample from $X$, we obtain arbitrarily small (positive) values more and more likely. In practice, such a situation may appear in a VAE setting, especially when we are dealing with outliers, which is a direct motivation for the improvement introduced in the next subsection.
\end{remark}

\subsection{Reducing Variational Gap}\label{subsec:reducing}

In this subsection, addressing the problem outlined in \Cref{rem:gamma}, we describe how to apply the importance sampling technique to obtain a tighter (additive) approximation of $f(\E X)$. It comes down to using (instead of $X$) the random variable 
\begin{equation}\label{eq:mean}
\begin{array}{l @{\;} l @{\;} l}
\overbar X_k & = & \frac{1}{k}(X_1+\ldots+X_k),
\end{array}
\end{equation}
representing the mean of a $k$-sample from $X$ (consisting of $k$ independent copies of $X$).
Clearly, $\E \overbar X_k=\E X$, 
which implies that $f(\E \overbar X_k)=f(\E X)$.
The following theorem is (in fact) a part of Theorem 1 from \citep{burda2015importance}, restated in a general setting. However, for completeness, we include novel and independent proof.

\begin{theorem}\label{th:burda}
Let $X$ be a random variable and $f$ be a continuous concave function. Then 
for every $k>0$ we have
\begin{equation}\label{eq:th:burda:1}
\begin{array}{l @{\;} l @{\;} l}
\E f(\overbar X_k) & \leq & \E f(\overbar X_{k+1}).
\end{array}
\end{equation}
Moreover, $f(\overbar X_k)$ converges to $f(\E X)$ almost surely, and
\begin{equation}\label{eq:th:burda:2}
\begin{array}{l @{\;} l @{\;} l}
\lim\limits_{k \to \infty} \E f(\overbar X_k) = f(\E X), 
\end{array}
\end{equation}
provided that the support of $X$ is contained in some compact interval lying in the domain of $f$.
\end{theorem}

\begin{proof}
Let $(\Omega,\mu)$ be a probabilistic space and $X\colon \Omega \to \R$ be a random variable. By the concavity of $f$ we directly conclude that
\begin{equation}\label{app_eq:th:burda:3}
\begin{array}{l @{\;} l @{\;} l}
f\left(\frac{1}{k+1}\sum_{i=1}^{k+1}x_i\right) & = & f\left(\frac{1}{k+1}\sum_{i=1}^{k+1}\frac{1}{k} \sum_{j=1,j\neq i}^{k+1} x_j\right)\\
& \geq & \frac{1}{k+1}\sum_{i=1}^{k+1}f\left(\frac{1}{k}\sum_{j=1,j\neq i}^{k+1} x_j\right).
\end{array}
\end{equation}
Then by monotonicity and linearity of the expected value, 
we obtain
\begin{equation}\label{app_eq:th:burda:4}
\begin{array}{@{}l @{\;} l @{\;} l}
\E f(\overbar X_{k+1}) & \geq & \frac{1}{k+1} \sum_{i=1}^{k+1}
\E f(\frac{1}{k} \sum_{j=1,j\neq i}^{k+1} X_j) = \E f(\overbar X_k),
\end{array}
\end{equation}
where $X_1,\ldots, X_{k+1}$ are independent copies of $X$. This gives the first assertion of the theorem.

Now consider the random variable $\overbar X_k$. From the strong law of large numbers, it follows that $\overbar X_k$ converges to $\E X$ almost surely. Since $f$ is continuous, this is also the case for $f(\overbar X_k)$ and $f(\E X)$. Hence, we conclude that $\E f(\overbar X_k)\to f(\E X)$ as $k \to \infty$, whenever the support of $X$ is contained in some compact interval lying in the domain of $f$, which completes the proof.
\end{proof}

Assuming bounded support for $X$ in \Cref{th:burda} we followed (Burda et al., 2015), where the respective theory behind the use of the importance sampling technique is based on this assumption and illustrates the underlying case in a simplified setting. Although, in some cases, the theorem is tending to ``survive'' without this restriction (see, e.g, \Cref{rem:gammasample}), the proof of a more general version would, however, cause some technical difficulties, resulting in reducing the clarity of the presentation.

Applying \eqref{eq:vargapint} to the random variable $\overbar X_n$, we conclude that
\begin{equation}\label{eq:varintred}
\begin{array}{l @{\;} l @{\;} l}
f(\E X) \in [\E f(\overbar X_k),\E f(\overbar X_k)+\E 
((\overbar Y_k-\overbar X_k) f'(\overbar X_k))],
\end{array}
\end{equation}
where $X$ and $Y$ are independent random variables with the same distribution, and $X_1,\ldots,X_k$ and $Y_1,\ldots,Y_k$ are independent copies of $X$ and $Y$, respectively. By~\Cref{th:burda}, the left end of the above interval converges to $f(\E X)$. Hence, a natural question arises, whether the same happens for the right end. In the following corollary, we show that this is the case.

\begin{corollary} \label{cl:3.2}
Let $X$ be a random variable and $f$ be a concave function. Assume also that $(m-x)f'(x)$ is convex in the support of $X$ for arbitrary constant $m$ that belongs to the support of $X$. Then the width
of the interval given in \eqref{eq:varintred}, i.e., $\E ((\overbar Y_k-\overbar X_k) f'(\overbar X_k))$, is a decreasing sequence with $k$.
Moreover, if the support of $X$ is contained in the closed bounded interval lying in the domain of $f'$, then the limit is 0.
\end{corollary}

\begin{proof}
It is enough to apply~\Cref{th:burda} for the concave function: $(x-\E X)f'(x)$.
\end{proof}
Now let us get back to the case of $f=\log$, where we directly obtain the following \mm{corollary.}

\begin{corollary}\label{cor:burdalog}
We have
\begin{equation}\label{eq:burdalog}
\begin{array}{l @{\;} l @{\;} l}
\E \log \overbar X_k &\leq & \log \E X \leq \E \log \overbar X_k+ \E\frac{\overbar Y_k}{\overbar X_k}-1.
\end{array}
\end{equation}
Moreover, $\E\frac{\overbar Y_k}{\overbar X_k}$ is a decreasing sequence with $k$, which converges to $1$, provided that the support of $X$ lies in some compact interval contained in $(0,\infty)$.
\end{corollary}

\begin{remark}\label{rem:gammasample}
As we have already outlined, in some cases we can use \Cref{th:burda} (and hence \Cref{cl:3.2,cor:burdalog}) to make the size of a variational gap \mm{arbitrarily} small, even if we cannot respectively bound values of $X$ almost surely.
Indeed, if we return to the example from \Cref{rem:gamma}, we easily see that $\overbar X_k\sim \text{Gamma}(ka,\theta/k)$ and $1/(\overbar X_k)\sim \text{Inv-Gamma}(ka,k/\theta)$. Therefore, assuming $ka>1$, we can calculate that
\begin{equation}\label{eq:betasample1}
\begin{array}{l @{\;} l @{\;} l}
\E \frac{\overbar Y_k}{\overbar X_k} -1  &  = & \frac{a\theta k}{\theta(ka-1)}-1=\frac{1}{ka-1}\xrightarrow{k\to \infty} 0.
\end{array}
\end{equation}
However, note that even though this means that we can find $k$ large enough to decrease a variational gap sufficiently, we  also see that when the value of $a$ approaches 0, we may be forced to wait for such an effect quite a long while increasing the value of $k$.
This is a direct motivation for the improvement introduced in the next subsection.
\end{remark}

\subsection{Improved Bounds for Variational Gap}

In this subsection, we provide another technique, which is crucial in the estimation of the size of the variational \mm{gap and} addresses the problem described in \Cref{rem:gammasample}.
Although we start with a general proposition, the idea of which lies in generalizing the \mm{bounds} obtained by the concavity, we eventually fix our attention on the case $f=\log$.

\begin{proposition} \label{pr:1}
Assume that $f$, $g$, and $h$ are arbitrary functions such that 
\begin{equation}\label{eq:pr:1:1}
\begin{array}{l @{\;} l @{\;} l}
f(a) & \leq & g(x)+a h(x) \text{ for every } a \text{ and }x.
\end{array}
\end{equation}
Then 
\begin{equation}\label{eq:pr:1:2}
\begin{array}{l @{\;} l @{\;} l}
f(\E X) & \leq & \E (g(X)+Y h(X)),
\end{array}
\end{equation}
where $X$ and $Y$ are two independent random variables with the same distribution.
\end{proposition}

\begin{proof}
The proof follows the same lines as the proof of~\Cref{th:1} \mm{(see \Cref{app:proofs})}.
We consider the probabilistic space $(\Omega,\mu)$ and independent random variables $X,Y\colon \Omega \to \R$ with the same distribution. We use the notation
\begin{equation}\label{eq:pr:1:3}
\begin{array}{l @{\;} l @{\;} l}
m & = &\E X=\int_{\Omega} X\, d\mu.
\end{array}
\end{equation}
Clearly $\E X=\E Y$.
Observe that, by the assumptions, for every $\omega \in \Omega$ we have
\begin{equation}\label{eq:pr:1:4}
\begin{array}{l @{\;} l @{\;} l}
f(m) & \leq & g(X(\omega))+m h(X(\omega)).
\end{array}
\end{equation}
Integrating the above formula over all $\omega \in \Omega$, we get
\begin{equation}\label{eq:pr:1:5}
\begin{array}{l @{\;} l @{\;} l}
f(\E X) &= &\int_{\Omega} f(m)\, d\mu 
\leq \int_{\Omega} g(X(\omega))+m h(X(\omega)) \,d\mu\\[2mm]
& = & \E g(X)+\E Y  \E  h(X) = \E(g(X)+Yh(X)),
\end{array}
\end{equation}
which ends the proof.
\end{proof}

Now we focus our attention on the case when $f=\log$. We prove that given an arbitrary function $g$, we can easily compute the optimal $h$.

\begin{lemma}\label{lem:1}
Let $g\colon (0,\infty) \to \R$ be an arbitrary function. Then
\begin{equation}\label{eq:lem:1:1}
\begin{array}{l @{\;} l @{\;} l}
\log a & \leq & g(x)+a \exp(-g(x)-1) \text{ for all }a,x>0.
\end{array}
\end{equation}
Moreover, for any function $h\colon (0,\infty) \to \R$ satisfying
\begin{equation}\label{eq:lem:1:2}
\begin{array}{l @{\;} l @{\;} l}
\log a & \leq & g(x)+a h(x),
\end{array}
\end{equation}
we have 
\begin{equation}\label{eq:lem:1:3}
\begin{array}{l @{\;} l @{\;} l}
h(x) & \geq & \exp(-g(x)-1).
\end{array}
\end{equation}
\end{lemma}

\begin{proof}
Consider an arbitrary function $h$. Let $x>0$ be fixed. We are going to find an equivalent condition for $h$ so that 
\begin{equation} \label{eq:a}
\begin{array}{l @{\;} l @{\;} l}
\log a & \leq & g(x)+a h(x) \text{ for all }a>0.
\end{array}
\end{equation}
Note that verifying \eqref{eq:a} is equivalent to checking whether
\begin{equation}\label{eq:lem:1:4}
\begin{array}{l @{\;} l @{\;} l}
h(x) & \geq & \frac{\log a - g(x)}{a} \text{ for all }a>0,
\end{array}
\end{equation}
or, equivalently,
\begin{equation}\label{eq:lem:1:5}
\begin{array}{l @{\;} l @{\;} l}
h(x) & \geq & \sup_{a>0}\frac{\log a - g(x)}{a}.
\end{array}
\end{equation}
One can easily check that if $w\colon (0,\infty)\to \R$ is a function defined as
\begin{equation}\label{eq:lem:w}
\begin{array}{l @{\;} l @{\;} l}
w(a) & = & \frac{\log a - g(x)}{a},
\end{array}
\end{equation}
 then \mbox{$w'(a)=\frac{1-(\log a -g(x))}{a^2}$}. Consequently, $w$ reaches the maximal value at \mbox{$a_x=\exp(1+g(x))$}. Thus, the equivalent condition for $h$ to satisfy \eqref{eq:a} is
\begin{equation}\label{eq:lem:1:6}
\begin{array}{l @{\;} l @{\;} l}
h(x) & \geq & w(a_x)=\exp(-1-g(x)),
\end{array}
\end{equation}
which proves all assertions.
\end{proof}

As a direct consequence of~\Cref{pr:1} and~\Cref{lem:1} we obtain the following theorem.

\begin{theorem} \label{th:3.7}
Let $g$ be an arbitrary function.
Then 
\begin{equation}\label{eq:euboimp}
\begin{array}{l @{\;} l @{\;} l}
\log\E X &\leq &\E (g(X)+Y \exp(-g(X)-1)),
\end{array}
\end{equation}
where $X$ and $Y$ are two positive independent random variables with the same distribution.
\end{theorem}

Now consider a one-parameter family of functions
\begin{equation}\label{eq:gc}
\begin{array}{l @{\;} l @{\;} l}
g_C(x) &= &\log x-1+C \; (C\in \R).
\end{array}
\end{equation}
Then by applying any function $g_C$ as $g$ in~\Cref{pr:1}, we obtain the following theorem.

\begin{theorem}\label{th:euboimpc}
Let $C$ be arbitrarily chosen. Then 
\begin{equation}\label{eq:th:euboimpc:1}
\begin{array}{l @{\;} l @{\;} l}
\E \log X \leq \log \E X &\leq & \E \log X-1+C+\exp(-C) \E\frac{Y}{X},
\end{array}
\end{equation}
where $X$ and $Y$ are two independent positive random variables with the same distribution.
\end{theorem}

Observe that \mm{increasing $C$ decreases} the last component of the \mm{right-hand} side formula  \mm{in \eqref{eq:th:euboimpc:1}}. Hence, the optimal value of $C$ (i.e., minimizing the upper bound for $\log \E X$) can be obtained as the one that minimizes the function $W(C)=C+\exp(-C) \E\frac{Y}{X}$. This easily leads us to the following corollary.

\begin{corollary} \label{cl:1}
Under the assumptions of~\Cref{th:euboimpc}, the optimal value of $C$ is $C=\log \E\frac{Y}{X}$, which gives the following estimates:
\begin{equation}\label{eq:col:euboimpc:1}
\begin{array}{l @{\;} l @{\;} l}
\E \log X & \leq & \log \E X \leq \E \log X +\log \E \frac{Y}{X}.
\end{array}
\end{equation}
\end{corollary}

Note that although the upper bound for $\log \E X$ provided by \eqref{eq:col:euboimpc:1} does not have an additive form, it tells us about the most optimal estimation we can obtain. Moreover, in this case, we can also apply the importance sampling technique resulting in reducing the variational gap, which is an immediate consequence of \Cref{cl:1,cor:burdalog}.

\begin{theorem}\label{cor:burdalogimproved}
We have
\begin{equation}\label{eq:burdalogimproved}
\begin{array}{l @{\;} l @{\;} l}
\E \log \overbar X_k &\leq & \log \E X \leq \E \log \overbar X_k+ \log \E\frac{\overbar Y_k}{\overbar X_k}.
\end{array}
\end{equation}
Moreover, $\log \E\frac{\overbar Y_k}{\overbar X_k}$ is a decreasing sequence with $k$, which converges to $0$, provided that the support of $X$ lies in some compact interval contained in $(0,\infty)$.
\end{theorem}
\begin{remark}\label{rem:betasamplelog}
Continuing the example involving the Gamma distributed random variable $X$ (see \Cref{rem:gamma,rem:gammasample}), which goes beyond the assumptions of \Cref{cor:burdalogimproved}, when $ka>1$ we have
\begin{equation}\label{eq:betasamplelog}
\begin{array}{l @{\;} l @{\;} l}
\log \E \frac{\overbar Y_k}{\overbar X_k}  &  = & \log \frac{ka}{ka-1}\xrightarrow{k\to \infty} 0.
\end{array}
\end{equation}
\mm{This means that in this case, the last conclusion of \Cref{cor:burdalogimproved} ``survives'', too.}
\end{remark}

\subsection{Quality of Estimations}

We have already proved (see \eqref{eq:vargapint}) that for any concave function~$f$ we have
\begin{equation}\label{eq:varint1}
\begin{array}{l @{\;} l @{\;} l}
f(\E X) & \in & [\E f(X), \E (f(X)+(Y-X)f'(X))],
\end{array}
\end{equation}
where $X$ and $Y$ are independent random variables with the same distribution. In this section, we are going to show that the optimal choice for an approximation of $f(\E X)$ is the middle of the above interval.

Let us start with the following general result, which relates to a special case of the delta method~\citep[][Section~5.3.1]{bickel2015mathematical}. For completeness, we include the proof in \mm{\Cref{app:proofs}}.

\begin{theorem}\label{th:approx}
Assume that $f$ is a smooth function. \mm{Let $X$ and $Y$ be} independent random variables with the same distribution, which attain only values $\varepsilon$-close to $\E X$, where $\varepsilon>0$ is small. Then
\begin{equation}\label{eq:th:approx:1}
\begin{array}{@{}l @{\;} l @{\;} l}
f(\E X) & = & E f(X)+\frac{1}{2} \E((Y-X)f'(X))+o(\varepsilon^2).
\end{array}
\end{equation}
\end{theorem}

\mm{Even though the assumptions of \Cref{th:approx} are somewhat unrealistic, it was formulated to improve readers' intuition. Note that the importance sampling technique leads to random variables being more and more concentrated around their means, which we utilize in our experiments using sufficiently large samples.}

Now we proceed to the case when $f=\log$ and we are going to apply the bounds for $\overbar{X}_k$.
Clearly, by the central limit theorem, for large $k$ the distribution of $\overbar{X}_k$ can be considered Gaussian. However, this approximation is not satisfactory from our point of view, since we are limited to the class of positive random variables (which \mm{are proper} arguments for the $\log$ function). Based on the results of \citep{mouri2013log},
it is known that typically the distribution of the mean of independent positive random variables can be better approximated by the log-normal distribution $\mathcal{LN}$. In the other words, we can write
$\overbar{X}_k \approx \mathcal{LN}(m,\sigma)$. In the following theorem, we prove that for a log-normal random variable $X$, the value of $\log \E X$ lies exactly in the middle of the interval given by~\Cref{cl:1}.

\begin{theorem}\label{th:lognormest}
Let $X$ and $Y$ be independent random variables with the same log-normal distribution. Then 
\begin{equation}\label{eq:lognormest}
\begin{array}{l @{\;} l @{\;} l}
\log \E X & = & \E \log X+\frac{1}{2} \log \E \frac{Y}{X}.
\end{array}
\end{equation}
\end{theorem}

\begin{proof}
Let $X\sim \mathcal{LN}(m,\sigma)$, which means that $\log X\sim \mathcal{N}(m,\sigma^2)$. Then $\E X=\exp(m+\sigma^2/2)$ and, consequently,
\begin{equation}\label{eq:th:lognormest:1}
\begin{array}{l @{\;} l @{\;} l}
\log \E X & = & m+\frac{\sigma^2}{2}, \; \E \log X=m.
\end{array}
\end{equation}
Moreover, $1/X\sim \mathcal{LN}(-m,\sigma)$ and hence
\begin{equation}\label{eq:th:lognormest:2}
\begin{array}{l @{\;} l @{\;} l}
\log \E\frac{Y}{X} &= &\log \E Y+\log  \E \frac{1}{X}
 = m+\frac{\sigma^2}{2}-m+\frac{\sigma^2}{2}=\sigma^2.
\end{array}
\end{equation}
By applying \eqref{eq:th:lognormest:1} and \eqref{eq:th:lognormest:2} in \eqref{eq:lognormest}, we obtain the conclusion.
\end{proof}

\begin{figure}[th!]
    \centering
    \begin{subfigure}{\columnwidth}
        \includegraphics[width=\linewidth]{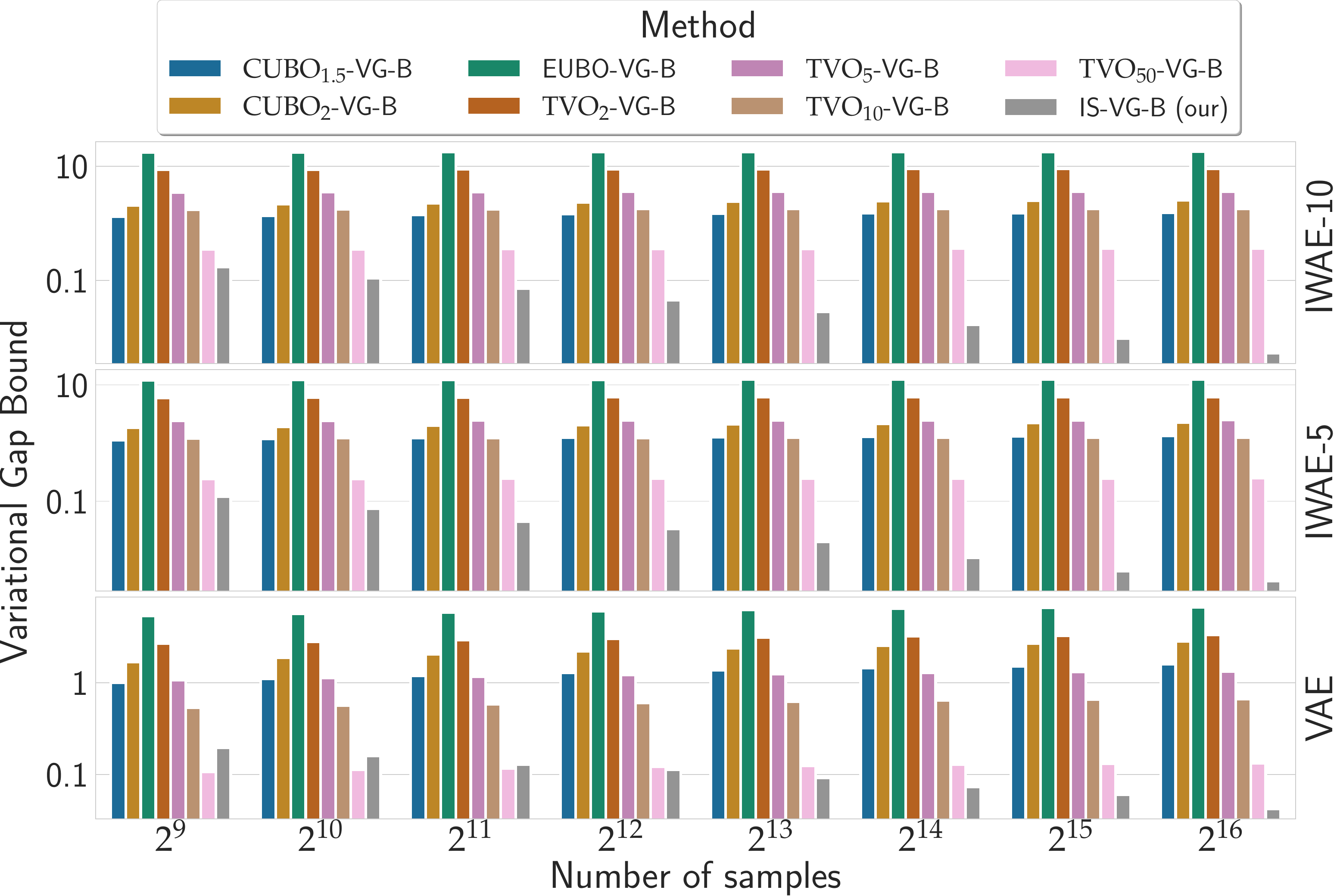}
        \caption{MNIST}
    \end{subfigure}
     \begin{subfigure}{\columnwidth}
        \includegraphics[width=\linewidth]{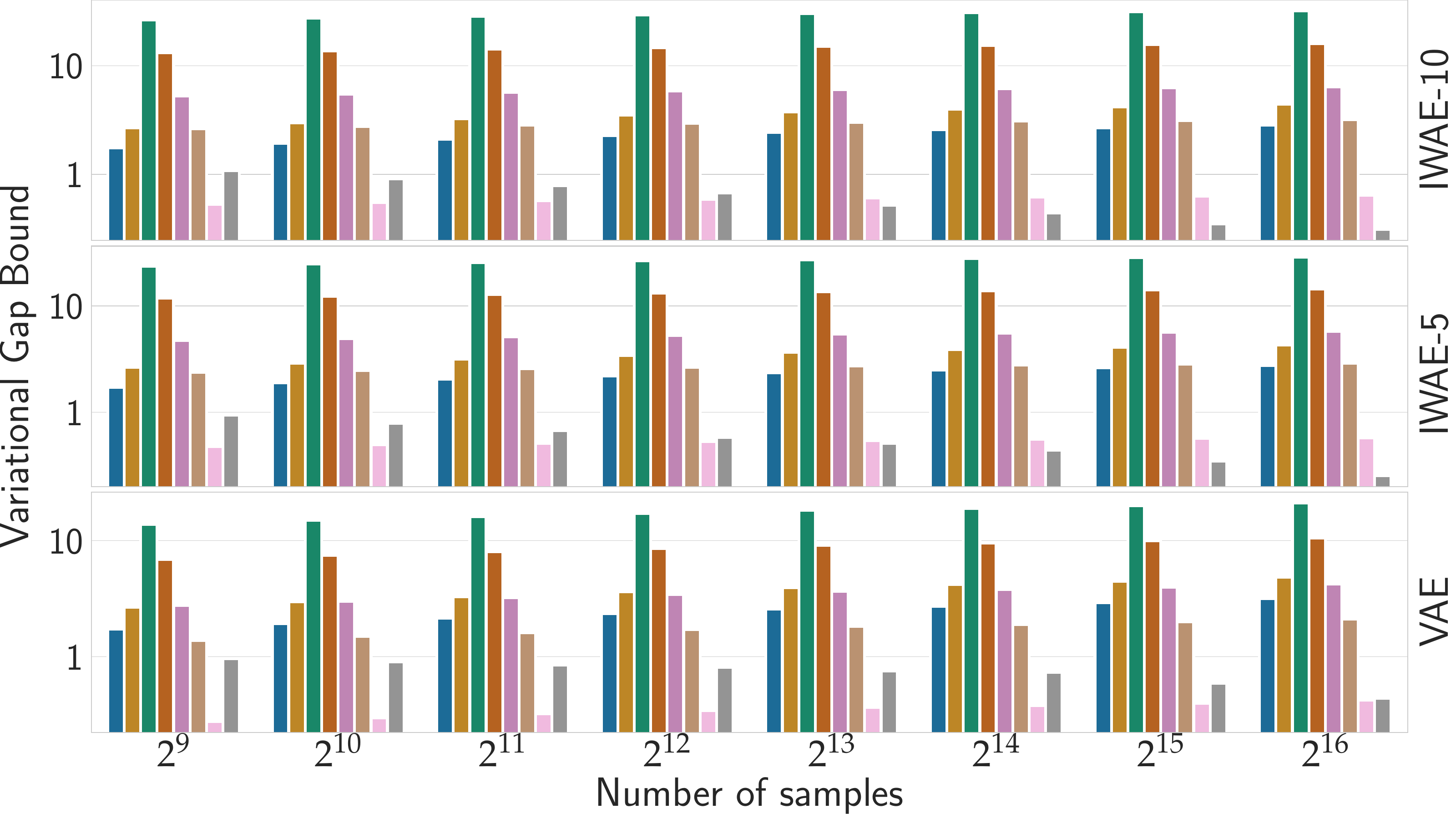}
        \caption{SVHN}
    \end{subfigure}
     \begin{subfigure}{\columnwidth}
         \includegraphics[width=\linewidth]{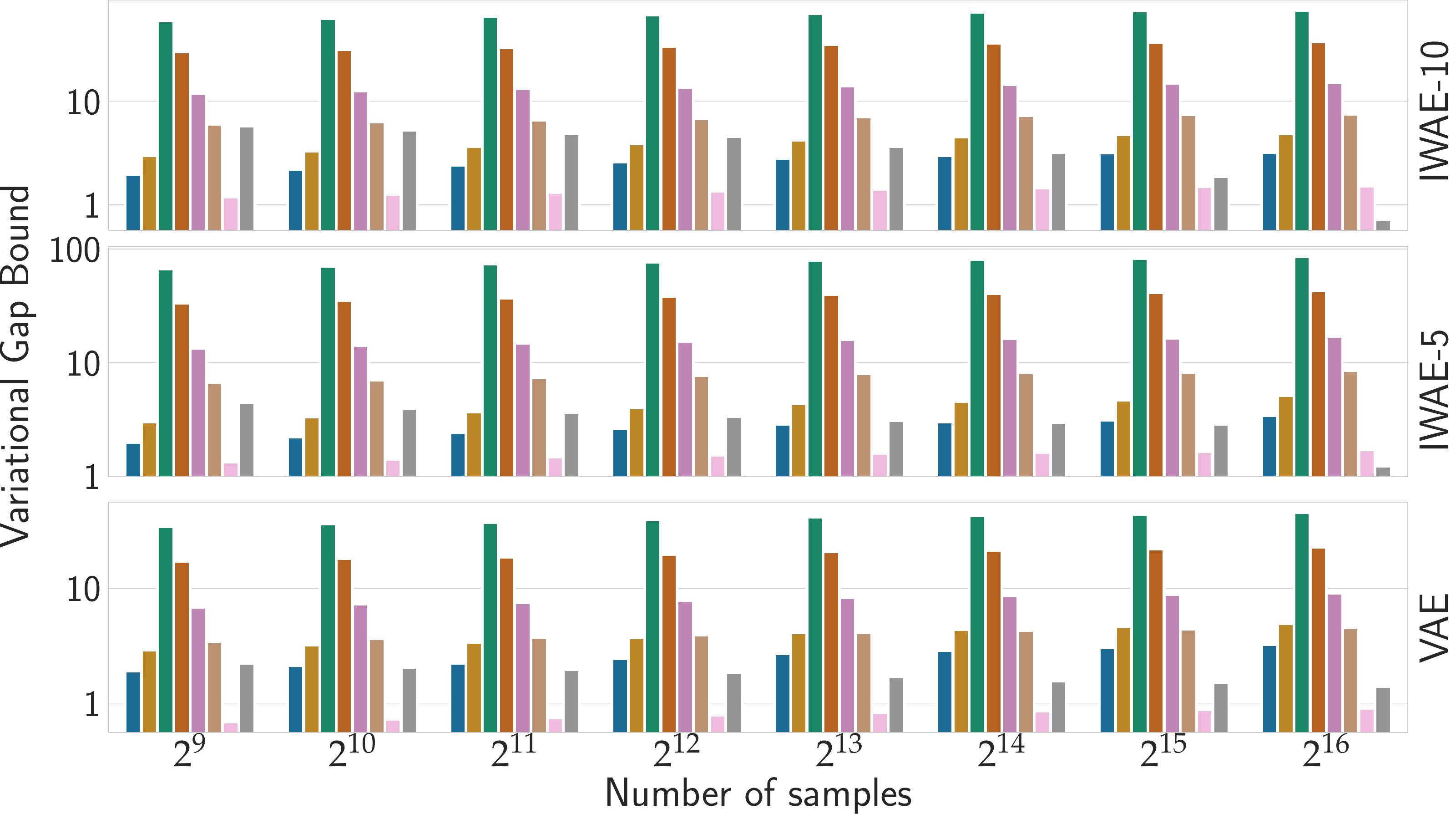}
         \caption{CelebA}
    \end{subfigure}
    \caption{\mm{Estimated size} of various variational gap bounds for the evidence of data (lower is better), calculated for VAE, IWAE-5, and IWAE-10 models, previously trained on MNIST, SVHN, and CelebA datasets, vs. \mm{the number} of latent samples.  All computations were averaged over 3 collections of samples, and over the test dataset.}
    \label{fig.compare_gap}
\end{figure}

\section{EXPERIMENTS}\label{sec:exp}

In this section, we consider variational generative models (like VAE). Let us first establish some standard notation. By $p(x)$ we denote the distribution \mm{induced by the model} on the input space $\mathcal{X}=\R^N$. By $p(z)$ we denote the prior distribution on the latent space $\mathcal{Z}$, while $q(z|x)$ denotes the variational encoder and $p(x|z)$ denotes the variational decoder.

Now, given a point $x \in \R^N$, its model log-likelihood (evidence) is \mm{expressed as follows:}
\begin{equation}\label{eq:exp:1}
\begin{array}{l @{\;} l @{\;} l}
\log p(x)
& = & \log \E_{z \sim q(\cdot |x)} \frac{p(x|z)p(z)}{q(z|x)}.
\end{array}
\end{equation}
To simplify the notation, we put
\begin{equation}\label{eq:exp:2}
\begin{array}{l @{\;} l @{\;} l}
R(x,z) & = & \frac{p(x|z)p(z)}{q(z|x)}.
\end{array}
\end{equation}
In the classical \mm{VAE,} we maximize 
\begin{equation}\label{eq:exp:3}
\begin{array}{l @{\;} l @{\;} l}
\text{ELBO} & = & \E_{z \sim q(\cdot |x)} \log R(x,z),
\end{array}
\end{equation}
which is the lower bound for the log-likelihood.

The idea behind the IWAE model \citep{burda2015importance} is to obtain an (asymptotically optimal) approximation of \mm{the} evidence by maximizing
\begin{equation}\label{eq:exp:4}
\begin{array}{l @{\;} l @{\;} l}
\text{IW-ELBO}_k=\E_{z_i \sim q(\cdot |x)} \log  \frac{1}{k}\sum_{i=1}^k R(x,z_i),
\end{array}
\end{equation}
which is a closer (than ELBO) lower bound for the log-likelihood. 

To obtain upper bounds for the evidence, it is enough to apply~ \Cref{th:euboimpc,cor:burdalogimproved} for $X=\frac{1}{k}\sum_{i=1}^k X_i$, where $X_i=R(x,z_i)$ and $Y_i=R(x,\tilde{z}_i)$, and all $z_i$ and $\tilde{z}_i$ are independently sampled from $q(\cdot|x)$. Then, we conclude that the size of the respective variational gap is bounded from above by the following value, which we will call {\em importance sampling variational gap bound (IS-VG-B)}: 
\begin{equation}\label{eq:isvgb}
\begin{array}{l @{\;} l @{\;} l}
\textrm{IS-VG-B} &= & C-1+\exp(-C)  \E_{z_i, \tilde z_i \sim q(\cdot |x)}  \frac{  \sum_{i=1}^k R(x,\tilde z_i)}{ \sum_{i=1}^k R(x,z_i)},
\end{array}
\end{equation}
where $C$ may be chosen \mm{arbitrarily}, with the optimum equal to
\begin{equation}\label{eq:cxopt}
\begin{array}{l @{\;} l @{\;} l}
C^{\,\text{opt}} &= &\log \E_{z_i, \tilde z_i \sim q(\cdot |x)}  \frac{  \sum_{i=1}^k R(x,\tilde z_i)}{ \sum_{i=1}^k R(x,z_i)}.
\end{array}
\end{equation}

Starting with a case study for simple
synthetic one-dimensional data generated from the Laplace
distribution, in the following few paragraphs we provide and discuss the results of the experiments, in which we compare our approach to those presented in \cite{dieng2017variational,ji2019stochastic,masrani2019thermodynamic}.
\mm{Mainly,} we apply \mm{all} considered estimation techniques for the variational gap to selected Gaussian autoencoders (i.e., classical VAE and two different IWAE models, on MNIST, SVHN, and CelebA datasets), previously learned using their own objectives and VAE experimental setup.

\paragraph{Case Study on Synthetic Data} \mm{Suppose that our data are drawn from a known distribution $p(x)$. Then obviously, the true evidence of data\footnote{In this paragraph, unlike before, the notion ``evidence'' refers to the expected value of the log-likelihood of data. In practice, this corresponds to taking the average over the whole dataset, which we use anyway in our experiments. Moreover, to avoid possible misunderstanding, we draw readers' attention to the fact that in our paper, the evidence of data is based on the model distribution (see \eqref{eq:evidence} and \eqref{eq:exp:1}). This paragraph is the only place where we also mention the true evidence.} is expressed as
$\int p(x)\log p(x) dx$.
Suppose that by training the VAE model, we construct the approximation of $p(x)$ in the class of distributions $p_\theta(x)$, where $\theta$ denotes the weights of the neural networks.
Then the (model) evidence is given as $\int p(x)\log p_{\theta}(x) dx$. 
Consequently, 
\begin{equation}\label{eq:ll}
\begin{array}{@{}l @{\;} l @{\;} l}
\text{true evidence} &=\text{evidence} + D_{\mathrm{KL}}(p(x)\|p_\theta(x)),
\end{array}
\end{equation}
where $D_{\mathrm{KL}}$ denotes the Kullback-Leibler divergence.
Obviously, if (which is a common case) $p(x)$ does not belong to the family of distributions $(p_\theta(x))_{\theta \in \Theta}$ (more precisely, $D_{KL}(p(x)\|p_\theta(x))>0$), then the evidence is less than the true-evidence.
Since we used the VAE model, we also have ELBO which is less than the evidence. Finally, we obtain upper and lower bounds for the evidence of data, which become tighter with the \mm{increasing number} of samples from the latent distribution $q_\theta(\cdot|x)$ (here $x$ represents any data point drawn from $p(x)$).

\begin{figure}[!ht]
    \centering
    \includegraphics[width=\columnwidth]{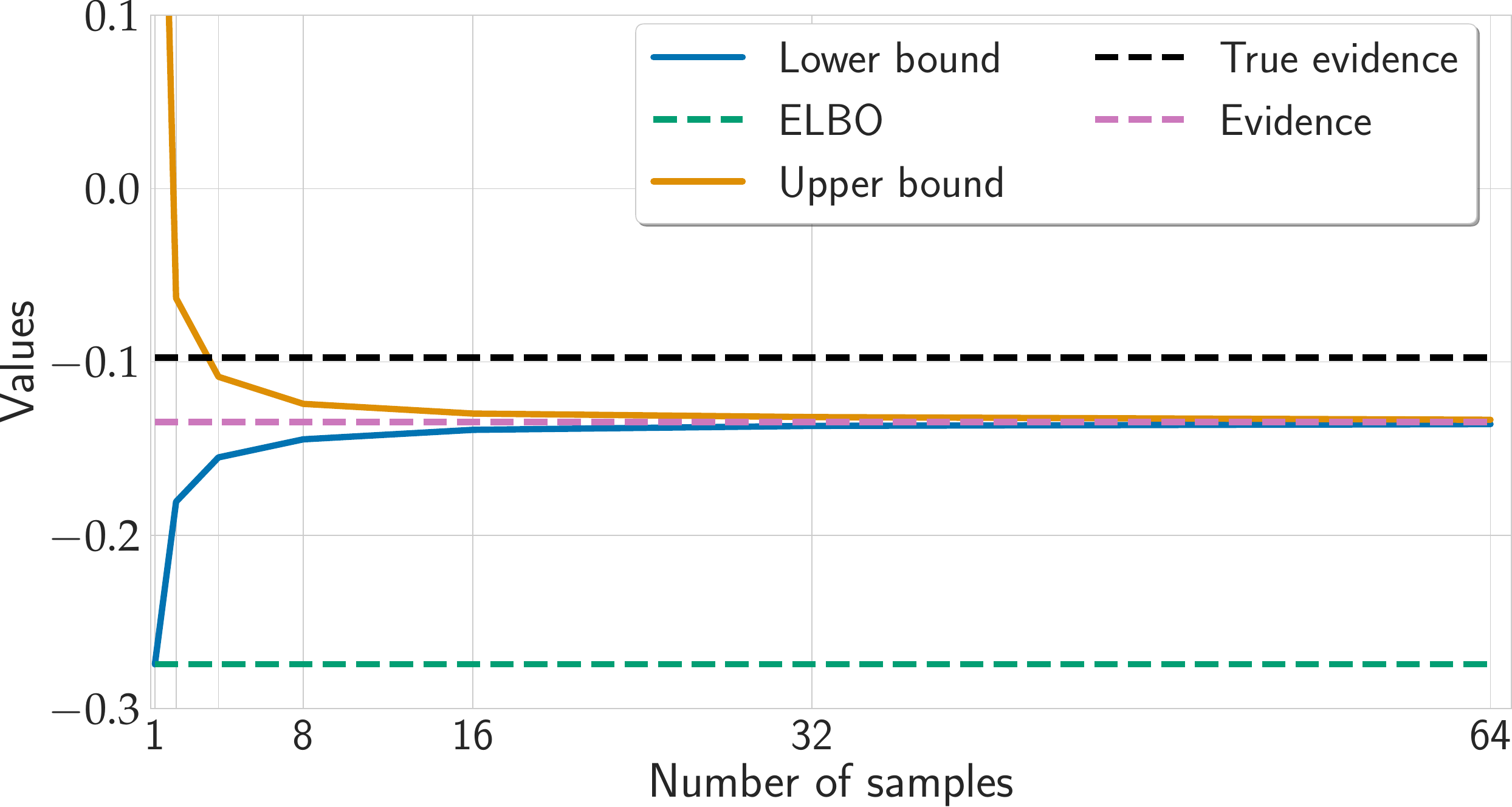}
    \caption{Behavior of lower and the upper bounds for the evidence vs.  \mm{the number} of samples from the latent.}
    \label{fig.laplace}
\end{figure}

\begin{figure}[th!]
    \centering
    \includegraphics[width=\linewidth]{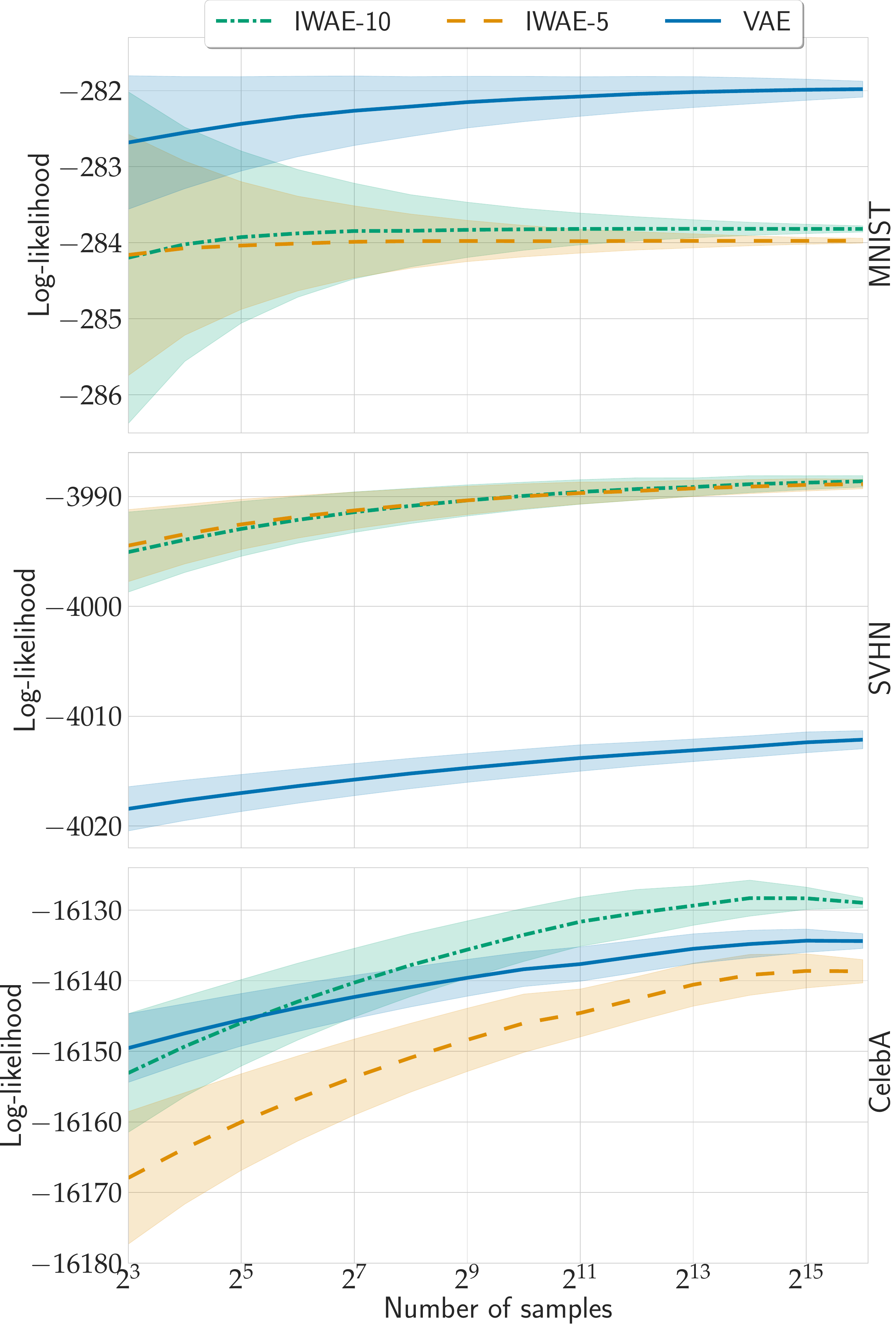}
    \caption{Behavior of lower and upper bounds for the evidence of data, calculated by our method for VAE, IWAE-5, and IWAE-10 models, previously trained on MNIST, SVHN, and CelebA datasets, vs. \mm{the number} of latent samples.  All computations were averaged over 3 collections of samples, and over the test dataset.
    }
    \label{fig.compare_models}
\end{figure}


We illustrate the above reasoning in the case of simple synthetic one-dimensional data generated from the Laplace distribution $\mathrm{Laplace}(0, 0.2)$. We trained the VAE model (for architecture details, see \Cref{app_app.add_results}) with one-dimensional latent space. The experimental results (see \Cref{fig.laplace}) are consistent with the \mm{above-mentioned} theoretical discussion.
The true evidence of the data coming from the Laplace distribution equals -0.097. On the other hand, ELBO (calculated as a value of the cost function of VAE) is a strong lower bound for the (model) evidence which, in turn, is less that the true evidence. Moreover, the IW-ELBO lower bound and our upper bound converge to the evidence of data, thus providing its tight estimate in the interval $[-0.137,-0.132]$,
with the ends calculated using $64$ latent samples.}

\paragraph{Experiments for VAE and IWAE Models} 
In our further experiments, we validate our importance sampling variational gap bound (given by \eqref{eq:isvgb} and \eqref{eq:cxopt}) \mm{against} the other state-of-the-art bounds, denoted by CUBO-VG-B, EUBO-VG-B, and TVO-VG-B, depending on the evidence upper bound used. 
All comparisons are made for VAE, IWAE-5, and IWAE-10 models, for which the respective bounds \mm{are calculated using (in total)} the same number of latent samples $z_i$. Each model was previously trained on three classical datasets, i.e., MNIST~\citep{lecun1998gradient}, SVHN~\citep{netzer2011reading}, and CelebA~\citep{liu2015deep}, using its own objective and VAE experimental setup (see \mm{\Cref{app_app.add_results}} for the details).
The code for all experiments is available on the GitHub repository \url{https://github.com/gmum/Bounding_Evidence_Estimating_LL}.

The obtained results are presented in \Cref{tab.gap,fig.compare_models,fig.compare_gap}. Namely, from \Cref{tab.gap} we learn that the proposed method for bounding the evidence of data is superior in comparison to all competitors, excluding $\text{TVO}_{50}$-VG-B \mm{for VAE} trained on \mm{SVHN and CelebA datasets}, as far as we make calculations using the largest considered number of latent samples (i.e., \mm{$2^{16}$}). \mm{However, estimating bounds for such sets of data is much more demanding because of complicated encoded data distributions. We believe increasing the sample size up to $2^{17}$ or $2^{18}$ (although quite time-consuming) would further weaken the impact of outliers and allow our method to definitely win.
Additionally,} a closer inspection of \Cref{fig.compare_gap} shows that our approach is the only one guaranteeing \mm{to decrease} a computed variational gap bound with \mm{the increasing} number of samples, which agrees with the conclusion of \Cref{cor:burdalogimproved}. This is due to applying the importance sampling technique, which (to our best knowledge) is not the case for the other methods \mm{(note that for them, the biased estimation may result in bounds growing with the sample size).}

\mm{More complete results of our approach, including the dependence of the obtained lower and upper bounds on the number of latent samples used in the computations, are presented in} \Cref{fig.compare_models}. In practice, they allow us (for any given dataset) to compare the effects of training between all considered models. Note that even though the IWAE models are learned using more latent samples, they \mm{do not} always deliver better results, i.e., greater values of the log-likelihood of data \citep[see also][]{rainforth2018tighter}. For example,
VAE trained on the MNIST dataset \mm{delivers} the best (the greatest) log-likelihood estimates\footnote{\mm{We would like to emphasize that, although our method is also able to evaluate the effects of even suboptimal training, we work on properly learned models. For example, the FID score for our Gaussian VAE equals 28.73 and is better than, e.g., 40.47 provided in \citep{JMLR:v21:19-560}, where a comparable experimental setup was used.
}}. 

\begin{figure}[th!]
    \centering
    \includegraphics[width=\linewidth]{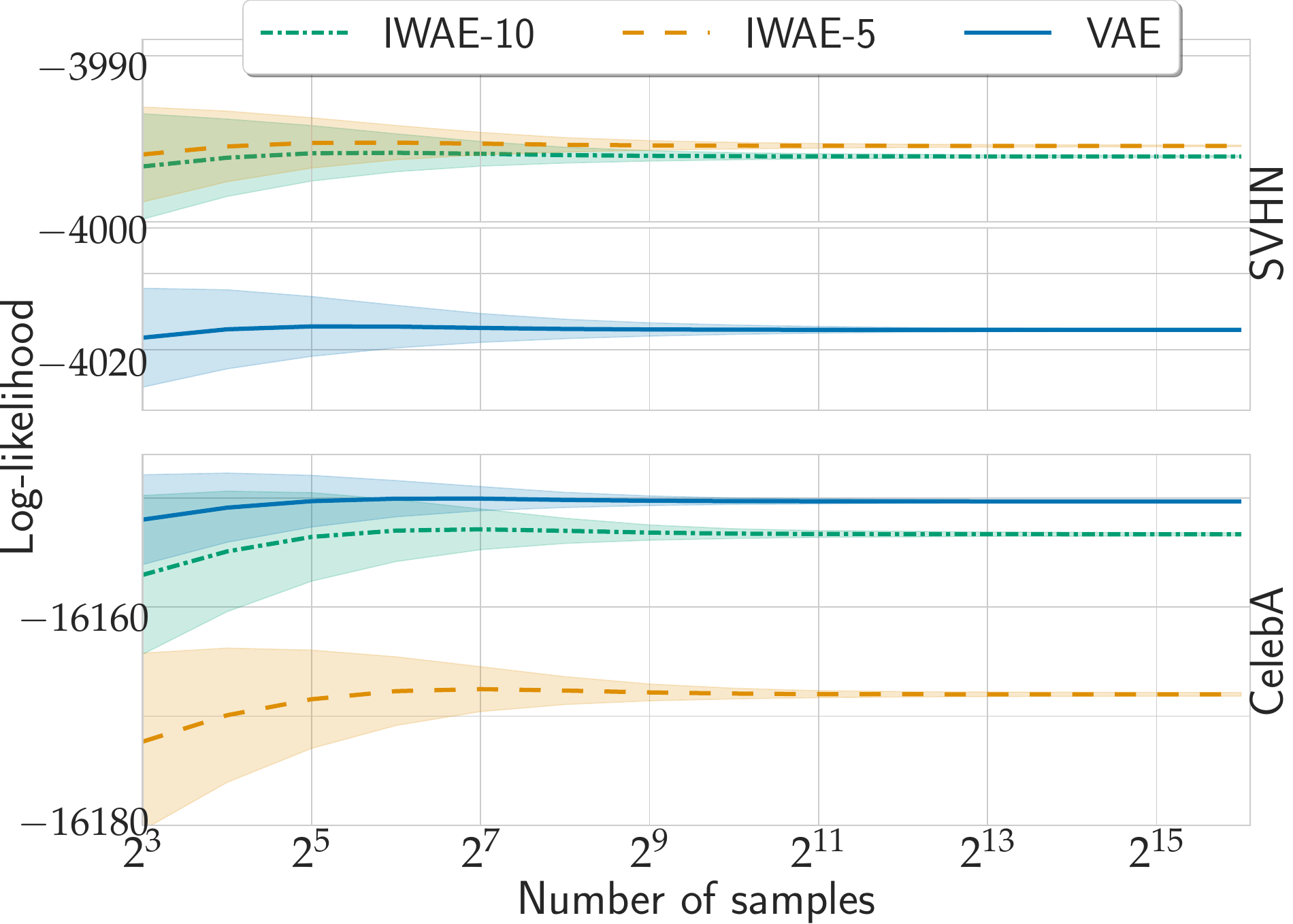}
    \caption{Behavior of lower and upper bounds for the evidence of data, calculated without 1\% outliers by our method for VAE, IWAE-5, and IWAE-10 models, previously trained on SVHN and CelebA datasets, vs. the number of latent samples.  All computations were averaged over \mm{3 collections of samples}, and over the test dataset.}
    \label{fig.gap_without_outliers}
\end{figure}

\mm{To prevent possible readers' concerns, we would like to explain that the gaps calculated on SVHN and CelebA datasets for different numbers of latent samples do not overlap because we use estimators (and not strict values, provided in \eqref{eq:isvgb} and \eqref{eq:cxopt}) suffering from the presence of outliers. To confirm this, we removed outliers and reran all computations. The obtained results are presented in \Cref{fig.gap_without_outliers}. Note that  
we observe no such phenomenon anymore, which supports our assertion.}

\begin{figure}[!ht]
    \centering
\includegraphics[width=0.96\linewidth]{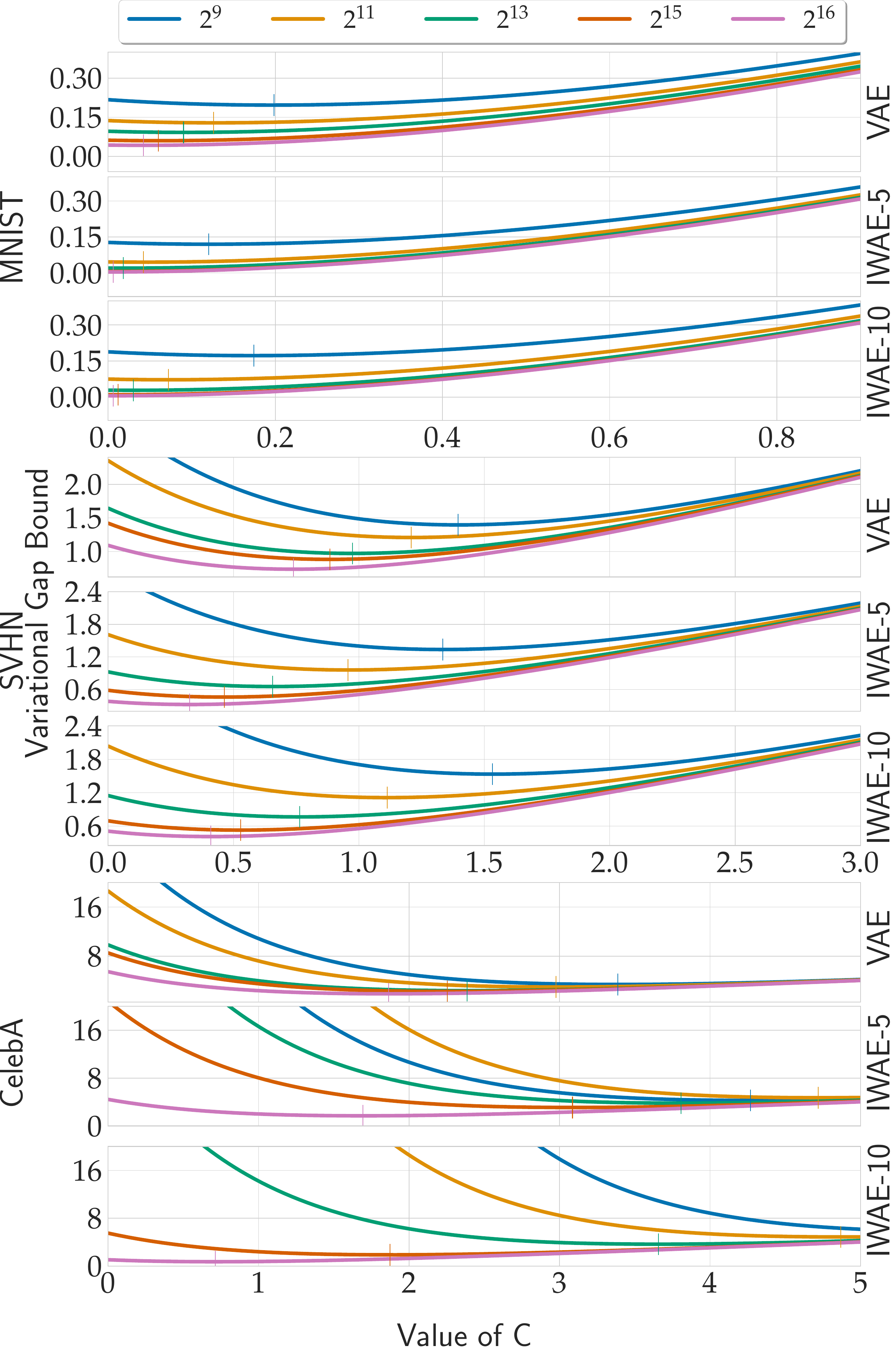}
    \caption{\mm{Estimated size} of the proposed variational gap bound, calculated using \eqref{eq:isvgb} for VAE, IWAE-5, and IWAE-10 models, previously trained on MNIST, SVHN, and CelebA datasets, vs. constant $C$. In each case, the optimal value $C^{\mathrm{opt}}$ is marked with a short vertical line. All computations were averaged over the test dataset.}
    \label{fig.C_gap}
\end{figure}

\mm{\paragraph{Dependence of Variational Gap Bound on $C$} We conducted additional experiments to explore the dependence of the proposed variational gap bound on the choice of constant $C$ in \eqref{eq:isvgb}. In \Cref{fig.C_gap} we present the results obtained for VAE, IWAE-5, and IWAE-10 models, previously trained on MNIST, SVHN, and CelebA datasets. Note that in each case, $C^{\mathrm{opt}}$ (see \eqref{eq:cxopt}) is a value for which IS-VG-B reaches a minimum.}

More details on the experimental results can be found in \mm{\Cref{app_app.add_results}}.

\section{CONCLUSION}

\mm{In this paper, we proposed novel upper bounds for the variational gap (see \eqref{eq:th:1} for a basic version and \eqref{eq:pr:1:2} for an improved version), and the use of the importance sampling technique to tighten them, which allowed us to calculate a precise estimation of $f(\E X)$ for any concave function $f$.
We focused particular attention on the most important case when $f=\log$ (see \eqref{eq:cor:1} for a basic version as well as \eqref{eq:euboimp} and \eqref{eq:th:euboimpc:1} for improved versions),} which led us to a practical method for validating the effects of training in generative models involving lower bound optimization. We conducted experiments that proved \mm{the superiority} of our approach in comparison to the other state-of-the-art techniques.

\paragraph{Limitation} \mm{In this contribution we did not consider possible  applications of our approach outside VAEs, which might be a natural direction for future work. Moreover,
 our} experience shows that the introduced novel bounds have rather limited utility in the training process.
\mm{Finally, calculating bounds in our experiments is based on estimators rather than strict values, which makes it sensitive to the presence of outliers, and results in nonrigorous bounds (although we note the discussed convergence results). However, we would like to emphasize that our proposed variational gap bounds are strict bounds as far as we limit ourselves to the theoretical approach (where we deal with random variables and their means).}

\paragraph{Societal Impact} There are no foreseen ethical or societal consequences for the research presented herein.

\subsubsection*{Acknowledgements}

This research was partially funded by the National Science Centre, Poland, grants no.  2020/39/D/ST6/01332 (work by \L{}ukasz Struski), 2021/43/B/ST6/01456 (work by Przemysław Spurek), and 2021/41/B/ST6/01370 (work by Jacek Tabor). Some experiments were performed on servers purchased with funds from the flagship project entitled ``Artificial Intelligence Computing Center Core Facility'' from the DigiWorld Priority Research Area within the Excellence Initiative -- Research University program at Jagiellonian University in Kraków.
Marcin Mazur's participation in the Conference was supported by the Faculty of Mathematics and Computer Science of the Jagiellonian University from the funds of the Excellence Initiative at the Jagiellonian University.

\bibliography{NIPS/ref}

\begin{thebibliography}{}

\bibitem[Abramovich and Persson, 2016]{abramovich2016some}
Abramovich, S. and Persson, L.-E. (2016).
\newblock Some new estimates of the `{Jensen} gap'.
\newblock {\em Journal of Inequalities and Applications}, 2016(1):1--9.

\bibitem[Bayer et~al., 2021]{bayer2021mind}
Bayer, J., Soelch, M., Mirchev, A., Kayalibay, B., and van~der Smagt, P.
  (2021).
\newblock Mind the gap when conditioning amortised inference in sequential
  latent-variable models.
\newblock {\em arXiv preprint arXiv:2101.07046}.

\bibitem[Bickel and Doksum, 2015]{bickel2015mathematical}
Bickel, P.~J. and Doksum, K.~A. (2015).
\newblock {\em Mathematical statistics: basic ideas and selected topics,
  volumes I-II package}.
\newblock Chapman and Hall/CRC.

\bibitem[Botvinick and Toussaint, 2012]{botvinick2012planning}
Botvinick, M. and Toussaint, M. (2012).
\newblock Planning as inference.
\newblock {\em Trends in Cognitive Sciences}, 16(10):485--488.

\bibitem[Brnetic et~al., 2015]{brnetic2015refinement}
Brnetic, I., Khan, K.~A., and Pecaric, J. (2015).
\newblock Refinement of {Jensen’s} inequality with applications to cyclic
  mixed symmetric means and cauchy means.
\newblock {\em J. Math. Inequal}, 9(4):1309--1321.

\bibitem[Burda et~al., 2015]{burda2015importance}
Burda, Y., Grosse, R., and Salakhutdinov, R. (2015).
\newblock Importance weighted autoencoders.
\newblock {\em arXiv preprint arXiv:1509.00519}.

\bibitem[Dayan and Hinton, 1997]{dayan1997using}
Dayan, P. and Hinton, G.~E. (1997).
\newblock Using expectation-maximization for reinforcement learning.
\newblock {\em Neural Computation}, 9(2):271--278.

\bibitem[Dieng et~al., 2017]{dieng2017variational}
Dieng, A.~B., Tran, D., Ranganath, R., Paisley, J., and Blei, D. (2017).
\newblock Variational inference via $\chi$ upper bound minimization.
\newblock {\em Advances in Neural Information Processing Systems}, 30.

\bibitem[Dragomir, 2013]{dragomir2013some}
Dragomir, S.~S. (2013).
\newblock Some reverses of the {Jensen} inequality with applications.
\newblock {\em Bulletin of the Australian Mathematical Society},
  87(2):177--194.

\bibitem[Duo, 2021]{duo2021improving}
Duo, X. (2021).
\newblock Improving actor-critic reinforcement learning via {Hamiltonian Monte
  Carlo} method.
\newblock In {\em Deep RL Workshop NeurIPS 2021}.

\bibitem[Gao et~al., 2017]{gao2017bounds}
Gao, X., Sitharam, M., and Roitberg, A.~E. (2017).
\newblock Bounds on the {Jensen} gap, and implications for mean-concentrated
  distributions.
\newblock {\em arXiv preprint arXiv:1712.05267}.

\bibitem[Grosse et~al., 2015]{grosse2015sandwiching}
Grosse, R.~B., Ghahramani, Z., and Adams, R.~P. (2015).
\newblock Sandwiching the marginal likelihood using bidirectional monte carlo.
\newblock {\em arXiv preprint arXiv:1511.02543}.

\bibitem[Horv{\'a}th, 2021]{horvath2021extensions}
Horv{\'a}th, L. (2021).
\newblock Extensions of recent combinatorial refinements of discrete and
  integral {Jensen} inequalities.
\newblock {\em Aequationes Mathematicae}, pages 1--21.

\bibitem[Ioffe and Szegedy, 2015]{ioffe2015batch}
Ioffe, S. and Szegedy, C. (2015).
\newblock Batch normalization: Accelerating deep network training by reducing
  internal covariate shift.
\newblock In {\em International Conference on Machine Learning}, pages
  448--456. PMLR.

\bibitem[Jebara and Pentland, 2001]{jebara2001reversing}
Jebara, T. and Pentland, A. (2001).
\newblock On reversing {Jensen's} inequality.
\newblock {\em Advances in Neural Information Processing Systems}, pages
  231--237.

\bibitem[Ji and Shen, 2019]{ji2019stochastic}
Ji, C. and Shen, H. (2019).
\newblock Stochastic variational inference via upper bound.
\newblock {\em arXiv preprint arXiv:1912.00650}.

\bibitem[Khan et~al., 2020]{khan2020new}
Khan, M.~A., Khan, S., and Chu, Y. (2020).
\newblock A new bound for the {Jensen} gap with applications in information
  theory.
\newblock {\em IEEE Access}, 8:98001--98008.

\bibitem[Kingma and Welling, 2013]{kingma2013auto}
Kingma, D.~P. and Welling, M. (2013).
\newblock Auto-encoding variational {Bayes}.
\newblock {\em arXiv preprint arXiv:1312.6114}.

\bibitem[Knop et~al., 2020]{JMLR:v21:19-560}
Knop, S., Spurek, P., Tabor, J., Podolak, I., Mazur, M., and Jastrzebski, S.
  (2020).
\newblock {Cramer-Wold} auto-encoder.
\newblock {\em Journal of Machine Learning Research}, 21(164):1--28.

\bibitem[LeCun et~al., 1998]{lecun1998gradient}
LeCun, Y., Bottou, L., Bengio, Y., and Haffner, P. (1998).
\newblock Gradient-based learning applied to document recognition.
\newblock {\em Proceedings of the IEEE}, 86(11):2278--2324.

\bibitem[Levine, 2018]{levine2018reinforcement}
Levine, S. (2018).
\newblock Reinforcement learning and control as probabilistic inference:
  Tutorial and review.
\newblock {\em arXiv preprint arXiv:1805.00909}.

\bibitem[Li and Turner, 2016]{NIPS2016_7750ca35}
Li, Y. and Turner, R.~E. (2016).
\newblock R\'{e}nyi divergence variational inference.
\newblock In {\em Advances in Neural Information Processing Systems},
  volume~29.

\bibitem[Liu et~al., 2015]{liu2015deep}
Liu, Z., Luo, P., Wang, X., and Tang, X. (2015).
\newblock Deep learning face attributes in the wild.
\newblock In {\em Proceedings of the IEEE International Conference on Computer
  Vision}, pages 3730--3738.

\bibitem[Maddison et~al., 2017]{maddison2017filtering}
Maddison, C.~J., Lawson, D., Tucker, G., Heess, N., Norouzi, M., Mnih, A.,
  Doucet, A., and Teh, Y.~W. (2017).
\newblock Filtering variational objectives.
\newblock {\em arXiv preprint arXiv:1705.09279}.

\bibitem[Masrani et~al., 2019]{masrani2019thermodynamic}
Masrani, V., Le, T.~A., and Wood, F. (2019).
\newblock The thermodynamic variational objective.
\newblock {\em Advances in Neural Information Processing Systems}, 32.

\bibitem[Mouri, 2013]{mouri2013log}
Mouri, H. (2013).
\newblock Log-normal distribution from a process that is not multiplicative but
  is additive.
\newblock {\em Physical Review E}, 88(4):042124.

\bibitem[Netzer et~al., 2011]{netzer2011reading}
Netzer, Y., Wang, T., Coates, A., Bissacco, A., Wu, B., and Ng, A.~Y. (2011).
\newblock Svhn: Reading digits in natural images with unsupervised feature
  learning.
\newblock {\em NIPS Workshop on Deep Learning and Unsupervised Feature
  Learning}.

\bibitem[Nielsen, 2010]{nielsen2010family}
Nielsen, F. (2010).
\newblock A family of statistical symmetric divergences based on {Jensen's}
  inequality.
\newblock {\em arXiv preprint arXiv:1009.4004}.

\bibitem[Nowozin, 2018]{nowozin2018debiasing}
Nowozin, S. (2018).
\newblock Debiasing evidence approximations: On importance-weighted
  autoencoders and jackknife variational inference.
\newblock In {\em International Conference on Learning Representations}.

\bibitem[Rainforth et~al., 2018]{rainforth2018tighter}
Rainforth, T., Kosiorek, A., Le, T.~A., Maddison, C., Igl, M., Wood, F., and
  Teh, Y.~W. (2018).
\newblock Tighter variational bounds are not necessarily better.
\newblock In {\em International Conference on Machine Learning}, pages
  4277--4285. PMLR.

\bibitem[Rezende et~al., 2014]{pmlr-v32-rezende14}
Rezende, D.~J., Mohamed, S., and Wierstra, D. (2014).
\newblock Stochastic backpropagation and approximate inference in deep
  generative models.
\newblock In {\em Proceedings of the 31st International Conference on Machine
  Learning}, volume~32 of {\em Proceedings of Machine Learning Research}, pages
  1278--1286.

\bibitem[Ruel and Ayres, 1999]{ruel1999jensen}
Ruel, J.~J. and Ayres, M.~P. (1999).
\newblock {Jensen’s} inequality predicts effects of environmental variation.
\newblock {\em Trends in Ecology \& Evolution}, 14(9):361--366.

\bibitem[Saeed et~al., 2022]{saeed2022refinements}
Saeed, T., Khan, M.~A., and Ullah, H. (2022).
\newblock Refinements of {Jensen’s} inequality and applications.
\newblock {\em AIMS Mathematics}, 7(4):5328--5346.

\bibitem[Todorov, 2008]{todorov2008general}
Todorov, E. (2008).
\newblock General duality between optimal control and estimation.
\newblock In {\em 2008 47th IEEE Conference on Decision and Control}, pages
  4286--4292. IEEE.

\bibitem[Toussaint and Storkey, 2006]{toussaint2006probabilistic}
Toussaint, M. and Storkey, A. (2006).
\newblock Probabilistic inference for solving discrete and continuous state
  {Markov} decision processes.
\newblock In {\em Proceedings of the 23rd International Conference on Machine
  Learning}, pages 945--952.

\bibitem[Williams et~al., 2017]{williams2017information}
Williams, G., Wagener, N., Goldfain, B., Drews, P., Rehg, J.~M., Boots, B., and
  Theodorou, E.~A. (2017).
\newblock Information theoretic mpc for model-based reinforcement learning.
\newblock In {\em 2017 IEEE International Conference on Robotics and Automation
  (ICRA)}, pages 1714--1721. IEEE.

\end{thebibliography}


\appendix
\onecolumn

\section{MISSING PROOFS}\label{app:proofs}

In this section, we provide proofs omitted from the main article.

\textbf{Theorem 1}. \label{app_th:1}
Let $f$ be a smooth concave function. Then
\begin{equation}\label{app_eq:th:1}
\begin{array}{l @{\;} l @{\;} l}
f(\E X) & \leq & \E (f(X)+(Y-X)f'(X)),
\end{array}
\end{equation}
where $X$ and $Y$ are two independent random variables with the same distribution.

\begin{proof}
Let $(\Omega,\mu)$ be a probabilistic space and $X,Y\colon \Omega \to \R$ be independent random variables with the same distribution. We use notation
\begin{equation}\label{eq:m=EX}
\begin{array}{l @{\;} l @{\;} l}
m & = &\E X=\int_{\Omega} X d\mu.
\end{array}
\end{equation}
Clearly $\E X=\E Y$.
Observe that applying Taylor's expansion and concavity of $f$ (which means that \mm{$f''\leq 0$}), for every $\omega \in \Omega$ we have
\begin{equation}\label{app_eq:taylor}
\begin{array}{l @{\;} l @{\;} l}
f(m) & \leq & f(X(\omega))+(m-X(\omega)) f'(X(\omega)).
\end{array}
\end{equation}
Integrating the above formula over all $\omega \in \Omega$ and making use \mm{of the} fact that $X$ and $Y$ are independent random variables with the same distribution, we obtain
\begin{equation}\label{app_eq:taylorn}
\begin{array}{l @{\;} l @{\;} l}
f(\E X) & = & \int_{\Omega} f(m) \, d\mu 
\leq \int_{\Omega} f(X(\omega))+(m-X(\omega)) f'(X(\omega)) \, d\mu\\[2mm]
& = & \E f(X)+\E Y  \E f'(X)-\E (X f'(X)) = \E(f(X)+(Y-X)f'(X)),
\end{array}
\end{equation}
which completes the proof.
\end{proof}

\textbf{Theorem 6}.
Assume that $f$ is a smooth function. \mm{Let $X$ and $Y$ be} independent random variables with the same distribution, which attain only values $\varepsilon$-close to $\E X$, where $\varepsilon>0$ is small. Then
\begin{equation}\label{app_eq:th:approx:1}
\begin{array}{@{}l @{\;} l @{\;} l}
f(\E X) & = & E f(X)+\frac{1}{2} \E((Y-X)f'(X))+o(\varepsilon^2).
\end{array}
\end{equation}

\begin{proof}
Let $(\Omega,\mu)$ be a probabilistic space and $X,Y\colon \Omega \to \R$ be random variables  satisfying the assumptions of the theorem. Put $m=\E X=\E Y$.
By applying Taylor's expansion for $f$ and $f'$, taking any $\omega\in \Omega$ we have
\begin{equation}\label{app_eq:taylor1}
f(m) = f(X(\omega))+(m-X(\omega)) f'(X(\omega)) + \frac{1}{2}(m-X(\omega))^2 f''(X(\omega))+o(\varepsilon^2)
\end{equation}
and
\begin{equation}\label{app_eq:taylor2}
\begin{array}{l @{\;} l @{\;} l}
f'(m) & = & f'(X(\omega))+(m-X(\omega)) f''(X(\omega)) +  o(\varepsilon).
\end{array}
\end{equation}
Hence, 
\begin{equation}\label{app_eq:taylor2n}
f(m) = f(X(\omega))+\frac{1}{2}(m-X(\omega))f'(m) +\frac{1}{2}(m-X(\omega))f'(X(\omega)) +  o(\varepsilon^2).
\end{equation}
Then integrating the above formula over all $\omega \in \Omega$ and making use \mm{of the} fact that $X$ and $Y$ are independent random variables with the same distribution, we obtain
\begin{equation}\label{app_eq:taylornn}
f(\E X) = \E f(X)+\frac{1}{2}(\E X \E f'(X)-\E (X f'(X))) + o(\varepsilon^2) 
= \E(f(X)+\frac{1}{2}(Y-X)f'(X)) + o(\varepsilon^2),
\end{equation}
which completes the proof.
\end{proof}

\section{OTHER BOUNDS}\label{app:bounds}
In this \mm{section,} we follow the notation established in \Cref{sec:exp} of the main paper.

\paragraph{$\chi$ Upper Bound (CUBO)}

The (general) $\chi$ upper bound for the evidence $\log p(x)$ was derived by \citet{dieng2017variational}, with the use of the $\chi$ divergence. It is given by the following formula:
\begin{equation}\label{app_cubo}
\text{CUBO}_n=\frac{1}{n}\log\E_{z\sim q(\cdot|x)}\left(\frac{p(x,z)}{q(z|x)}\right)^n=\frac{1}{n}\log\E_{z\sim q(\cdot|x)}R(x,z)^n,
\end{equation}
where $n\geq 1$. Let us note that CUBO is strictly related to the variational R\'{e}nyi \mm{bound provided} by \citet{NIPS2016_7750ca35}.

The authors of \citep{dieng2017variational} prove (see the sandwiching theorem therein) that $\text{CUBO}_n$ is a \mm{non-decreasing} function of $n\geq 1$ and
\begin{equation}\label{app_sandwich}
\text{IW-ELBO}_k \leq \log p(x)\leq \text{CUBO}_n
\end{equation}
for any $k\geq 1$. Hence, for comparison with our approach, as the respective bound for the variational gap (CUBO-VG-B) we take a difference between $\text{CUBO}_n$ and $\text{IW-ELBO}_k$, i.e.
\begin{equation}\label{app_cubo-vg-b}
\text{CUBO}_n\text{-VG-B}=\text{CUBO}_n-\text{IW-ELBO}_k.
\end{equation}

\paragraph{Thermodynamic Variational Objective (TVO)}

\citet{masrani2019thermodynamic} provide upper and lower bounds for the evidence \mm{$\log p(x)$,} obtained by applying \mm{the thermodynamic} integration technique. It is based on bounding (via left and right Riemann sum) a one-dimensional integral of \mm{the expected values of instantaneous} ELBO, calculated under latent distributions $\pi_\beta$ for $\beta\in [0,1]$, lying on a geometric path between $q(z|x)$ and $p(x,z)$ (hence $\beta_0=q(z|x)$ and $\beta_1=p(z|x)$). This leads to a lower and upper version of the thermodynamic variational objective (TVO), given by the following formulas:
\begin{equation}\label{app_tvol}
\text{TVO}^L_K=\frac{1}{K}\sum_{l=0}^{K-1}\E_{z\sim \pi_{\beta_l}} \log \frac{p(x,z)}{q(z|x)} = \frac{1}{K}\sum_{l=0}^{K-1}\E_{z\sim \pi_{\beta_l}} \log R(x,z),
\end{equation}
and
\begin{equation}\label{app_tvou}
\text{TVO}^U_K=\frac{1}{K}\sum_{l=1}^{K}\E_{z\sim \pi_{\beta_l}} \log \frac{p(x,z)}{q(z|x)}=\frac{1}{K}\sum_{l=1}^{K}\E_{z\sim\pi_{\beta_l}} \log R(x,z),
\end{equation}
where $\beta_l=l/K$. Let us note that $\text{TVO}^L_1=\text{ELBO}$ and $\text{TVO}^U_1$ \mm{coincides} with the evidence upper bound (EUBO), which was introduced by \citet{ji2019stochastic}.

It was proved in \citep{masrani2019thermodynamic} that for any $K\geq 1$ we have
\begin{equation}\label{app_tvobound}
\text{TVO}^L_K\leq \log p(x)\leq  \text{TVO}^U_K.
\end{equation}
 Hence, for our purpose, as the respective bound for the variational gap (TVO-VG-B) we take a difference between $\text{TVO}^U_K$ and $\text{TVO}^L_K$, i.e.:
\begin{equation}\label{app_tvo-vg-b}
\text{TVO}_K\text{-VG-B}=\text{TVO}^U_K-\text{TVO}^L_K.
\end{equation}
In particular,  \begin{equation}\label{app_eubo-vg-b}
\text{EUBO-VG-B}=\text{TVO}_1\text{-VG-B}=\text{TVO}^U_1-\text{TVO}^L_1=\text{EUBO}-\text{ELBO}.
\end{equation}

\section{EXPERIMENTAL DETAILS AND EXTENSIONS}\label{app_app.add_results}

\paragraph{Experimental Results for VAE and IWAE Models} 
We validated our importance sampling variational gap bound (IS-VG-B) \mm{against} the other state-of-the-art bounds, denoted by CUBO-VG-B, EUBO-VG-B, and TVO-VG-B, depending on the evidence upper bound used.
All computations were made for VAE and IWAE models, previously trained with the use of ELBO, $\text{IW-ELBO}_5$, and $\text{IW-ELBO}_{10}$ objectives. For the IWAE models, we used the same neural architectures as in VAE.

\Cref{app_tab.whole_gap}, which is an extension of \Cref{tab.gap} \mm{from} the main paper, presents the size of all considered variational gap bounds for the evidence of data, computed with the use of \mm{various numbers} of latent samples and averaged over the test dataset.

\begin{table}[!ht]\small
\centering
\caption{\mm{Estimated size} of various variational gap bounds for the evidence of data (lower is better), calculated for VAE, IWAE-5, and IWAE-10 models, previously trained on MNIST, SVHN, and CelebA datasets. All computations were averaged over 3 collections of $2^3$--$2^{16}$ latent samples, and over the test dataset.}
\label{app_tab.whole_gap}
\begin{tabular}{@{}c@{}c@{\,}c@{\,}c@{\,}c@{\,}c@{\,}c@{\,}c@{\,}c@{\;}c@{\;\;}c@{\,}c@{\,}c@{\,}c@{\,}c@{\,}c@{\,}c@{\;}c@{\;\;}c@{\,}c@{\,}c@{\,}c@{\,}c@{\,}c@{\,}c@{\;}c@{}}
\toprule
& & \multicolumn{24}{c}{\bfseries VARIATIONAL GAP BOUND} \\
\cmidrule{3-26}
 \rotatebox{90}{\bfseries Model} & \rotatebox{90}{\bfseries \#Samples} & \rotatebox{90}{\bfseries IS (OUR)} & \rotatebox{90}{\bfseries $\text{CUBO}_{1.5}$} & \rotatebox{90}{\bfseries $\text{CUBO}_2$} & \rotatebox{90}{\bfseries EUBO} & \rotatebox{90}{\bfseries $\text{TVO}_2$} & \rotatebox{90}{\bfseries $\text{TVO}_5$} & \rotatebox{90}{\bfseries $\text{TVO}_{10}$} & \rotatebox{90}{\bfseries $\text{TVO}_{50}$}
 & \rotatebox{90}{\bfseries IS (OUR)} & \rotatebox{90}{\bfseries $\text{CUBO}_{1.5}$} & \rotatebox{90}{\bfseries $\text{CUBO}_2$} & \rotatebox{90}{\bfseries EUBO} & \rotatebox{90}{\bfseries $\text{TVO}_2$} & \rotatebox{90}{\bfseries $\text{TVO}_5$} & \rotatebox{90}{\bfseries $\text{TVO}_{10}$} & \rotatebox{90}{\bfseries $\text{TVO}_{50}$}
 & \rotatebox{90}{\bfseries IS (OUR)} & \rotatebox{90}{\bfseries $\text{CUBO}_{1.5}$} & \rotatebox{90}{\bfseries $\text{CUBO}_2$} & \rotatebox{90}{\bfseries EUBO} & \rotatebox{90}{\bfseries $\text{TVO}_2$} & \rotatebox{90}{\bfseries $\text{TVO}_5$} & \rotatebox{90}{\bfseries $\text{TVO}_{10}$} & \rotatebox{90}{\bfseries $\text{TVO}_{50}$}
 \\
\midrule
& & \multicolumn{8}{c}{\bfseries MNIST} & \multicolumn{8}{c}{\bfseries SVHN} & \multicolumn{8}{c}{\bfseries CelebA} \\
\cmidrule(r){3-10} \cmidrule(r){11-18} \cmidrule(r){19-26}
\multirow[c]{13}{*}{\rotatebox{90}{VAE}} 
& $2^{3}$ & 0.76 & 0.37 & 0.59 & 2.83 & 1.42 & 0.57 & 0.28 & 0.06  & 1.61 & 0.53 & 0.82 & 5.67 & 2.83 & 1.13 & 0.57 & 0.11 & 4.76 & 0.63 & 0.95 & 16.66 & 8.33 & 3.33 & 1.67 & 0.33 \\
& $2^{4}$ & 0.58 & 0.48 & 0.78 & 3.43 & 1.72 & 0.69 & 0.34 & 0.07   & 1.45 & 0.72 & 1.12 & 7.23 & 3.62 & 1.45 & 0.72 & 0.14  & 4.00 & 0.84 & 1.27 & 20.43 & 10.21 & 4.08 & 2.04 & 0.41 \\
& $2^{5}$ & 0.45 & 0.59 & 0.97 & 3.94 & 1.97 & 0.79 & 0.39 & 0.08   & 1.31 & 0.92 & 1.42 & 8.70 & 4.35 & 1.74 & 0.87 & 0.17  & 3.46 & 1.05 & 1.59 & 23.80 & 11.90 & 4.76 & 2.38 & 0.48 \\
& $2^{6}$ & 0.34 & 0.70 & 1.15 & 4.37 & 2.19 & 0.87 & 0.44 & 0.09   & 1.20 & 1.12 & 1.72 & 10.07 & 5.03 & 2.01 & 1.01 & 0.20 & 3.06 & 1.26 & 1.92 & 26.78 & 13.39 & 5.36 & 2.68 & 0.54 \\
& $2^{7}$ & 0.29 & 0.81 & 1.34 & 4.75 & 2.37 & 0.95 & 0.47 & 0.09   & 1.13 & 1.32 & 2.03 & 11.35 & 5.68 & 2.27 & 1.14 & 0.23 & 2.78 & 1.48 & 2.24 & 29.37 & 14.69 & 5.87 & 2.94 & 0.59 \\
& $2^{8}$ & 0.23 & 0.91 & 1.51 & 5.08 & 2.54 & 1.02 & 0.51 & 0.10   & 1.03 & 1.52 & 2.35 & 12.61 & 6.31 & 2.52 & 1.26 & 0.25 & 2.50 & 1.69 & 2.56 & 31.83 & 15.91 & 6.37 & 3.18 & 0.64 \\
& $2^{9}$ & 0.20 & 1.01 & 1.70 & 5.38 & 2.69 & 1.08 & 0.54 & 0.11   & 0.96 & 1.73 & 2.66 & 13.81 & 6.90 & 2.76 & 1.38 & 0.28 & 2.25 & 1.90 & 2.89 & 34.11 & 17.05 & 6.82 & 3.41 & 0.68 \\
& $2^{10}$ & 0.17 & 1.10 & 1.87 & 5.63 & 2.82 & 1.13 & 0.56 & 0.11  & 0.90 & 1.94 & 2.97 & 14.95 & 7.48 & 2.99 & 1.49 & 0.30 & 2.13 & 2.12 & 3.21 & 36.17 & 18.08 & 7.23 & 3.62 & 0.72 \\
& $2^{11}$ & 0.13 & 1.20 & 2.04 & 5.87 & 2.94 & 1.17 & 0.59 & 0.12  & 0.85 & 2.15 & 3.30 & 16.09 & 8.05 & 3.22 & 1.61 & 0.32 & 2.05 & 2.23 & 3.37 & 37.29 & 18.64 & 7.46 & 3.73 & 0.74 \\
& $2^{12}$ & 0.10 & 1.28 & 2.21 & 6.08 & 3.04 & 1.22 & 0.61 & 0.12  & 0.81 & 2.35 & 3.62 & 17.18 & 8.59 & 3.44 & 1.72 & 0.34 & 1.94 & 2.44 & 3.70 & 39.20 & 19.60 & 7.84 & 3.92 & 0.79 \\
& $2^{13}$ & 0.09 & 1.37 & 2.38 & 6.27 & 3.13 & 1.25 & 0.63 & 0.13  & 0.76 & 2.57 & 3.94 & 18.25 & 9.12 & 3.65 & 1.82 & 0.36 & 1.77 & 2.69 & 4.08 & 41.28 & 20.64 & 8.26 & 4.13 & 0.83 \\
& $2^{14}$ & 0.08 & 1.45 & 2.54 & 6.42 & 3.21 & 1.28 & 0.64 & 0.13  & 0.73 & 2.73 & 4.18 & 19.02 & 9.51 & 3.80 & 1.90 & 0.38 & 1.71 & 2.85 & 4.32 & 42.49 & 21.25 & 8.50 & 4.25 & 0.85 \\
& $2^{15}$ & 0.06 & 1.52 & 2.68 & 6.56 & 3.28 & 1.31 & 0.66 & 0.13  & 0.59 & 2.92 & 4.48 & 19.94 & 9.97 & 3.99 & 1.99 & 0.40 & 1.25 & 3.02 & 4.59 & 43.75 & 21.88 & 8.75 & 4.38 & 0.88 \\
& $2^{16}$ & 0.04 & 1.59 & 2.83 & 6.68 & 3.34 & 1.34 & 0.67 & 0.13  & 0.44 & 3.16 & 4.84 & 21.08 & 10.54 & 4.22 & 2.11 & 0.42 & 1.12 & 3.25 & 4.92 & 45.45 & 22.73 & 9.09 & 4.55 & 0.91 \\
\midrule
\multirow[c]{13}{*}{\rotatebox{90}{IWAE-5}} 
& $2^{3}$ & 1.76 & 0.54 & 0.83 & 8.73 & 4.36 & 1.75 & 0.87 & 0.17 & 3.06 & 0.60 & 0.92 & 13.30 & 6.65 & 2.66 & 1.33 & 0.27          & 9.54 & 0.66 & 1.00 & 35.33 & 17.67 & 7.07 & 3.53 & 0.71  \\         
& $2^{4}$ & 1.11 & 0.68 & 1.06 & 9.84 & 4.92 & 1.97 & 0.98 & 0.20    & 2.40 & 0.80 & 1.22 & 15.85 & 7.93 & 3.17 & 1.59 & 0.32       & 8.04 & 0.88 & 1.33 & 42.57 & 21.29 & 8.51 & 4.26 & 0.85  \\         
& $2^{5}$ & 0.69 & 0.81 & 1.26 & 10.62 & 5.31 & 2.12 & 1.06 & 0.21   & 1.95 & 0.99 & 1.51 & 17.94 & 8.97 & 3.59 & 1.79 & 0.36       & 6.85 & 1.11 & 1.67 & 48.78 & 24.39 & 9.76 & 4.88 & 0.98  \\         
& $2^{6}$ & 0.43 & 0.91 & 1.44 & 11.15 & 5.58 & 2.23 & 1.12 & 0.22   & 1.59 & 1.18 & 1.80 & 19.72 & 9.86 & 3.94 & 1.97 & 0.39       & 6.01 & 1.33 & 2.00 & 54.09 & 27.05 & 10.82 & 5.41 & 1.08 \\         
& $2^{7}$ & 0.27 & 1.00 & 1.59 & 11.53 & 5.76 & 2.31 & 1.15 & 0.23   & 1.31 & 1.36 & 2.08 & 21.23 & 10.62 & 4.25 & 2.12 & 0.42      & 5.37 & 1.55 & 2.33 & 58.81 & 29.41 & 11.76 & 5.88 & 1.17 \\         
& $2^{8}$ & 0.19 & 1.07 & 1.72 & 11.78 & 5.89 & 2.36 & 1.18 & 0.24   & 1.10 & 1.54 & 2.36 & 22.53 & 11.27 & 4.51 & 2.25 & 0.45      & 4.82 & 1.77 & 2.66 & 63.10 & 31.55 & 12.62 & 6.31 & 1.26 \\         
& $2^{9}$ & 0.12 & 1.13 & 1.82 & 11.96 & 5.98 & 2.39 & 1.20 & 0.24   & 0.94 & 1.71 & 2.63 & 23.68 & 11.84 & 4.74 & 2.37 & 0.47      & 4.41 & 1.99 & 3.00 & 66.91 & 33.46 & 13.38 & 6.69 & 1.34 \\         
& $2^{10}$ & 0.07 & 1.17 & 1.91 & 12.09 & 6.05 & 2.42 & 1.21 & 0.24  & 0.79 & 1.88 & 2.90 & 24.71 & 12.35 & 4.94 & 2.47 & 0.49      & 3.98 & 2.21 & 3.33 & 70.48 & 35.24 & 14.10 & 7.05 & 1.41 \\         
& $2^{11}$ & 0.04 & 1.21 & 1.99 & 12.18 & 6.09 & 2.44 & 1.22 & 0.24  & 0.67 & 2.05 & 3.16 & 25.61 & 12.81 & 5.12 & 2.56 & 0.51      & 3.60 & 2.43 & 3.67 & 73.79 & 36.90 & 14.76 & 7.38 & 1.48 \\         
& $2^{12}$ & 0.03 & 1.24 & 2.05 & 12.25 & 6.12 & 2.45 & 1.22 & 0.24  & 0.58 & 2.20 & 3.41 & 26.41 & 13.21 & 5.28 & 2.64 & 0.53      & 3.34 & 2.65 & 4.00 & 76.81 & 38.40 & 15.36 & 7.68 & 1.54 \\         
& $2^{13}$ & 0.02 & 1.27 & 2.11 & 12.30 & 6.15 & 2.46 & 1.23 & 0.25  & 0.51 & 2.35 & 3.65 & 27.14 & 13.57 & 5.43 & 2.71 & 0.54      & 3.09 & 2.87 & 4.33 & 79.60 & 39.80 & 15.92 & 7.96 & 1.59 \\         
& $2^{14}$ & 0.01 & 1.29 & 2.16 & 12.33 & 6.17 & 2.47 & 1.23 & 0.25  & 0.44 & 2.49 & 3.88 & 27.77 & 13.89 & 5.55 & 2.78 & 0.56      & 2.97 & 3.01 & 4.54 & 81.15 & 40.57 & 16.23 & 8.11 & 1.62 \\         
& $2^{15}$ & 0.01 & 1.31 & 2.20 & 12.36 & 6.18 & 2.47 & 1.24 & 0.25  & 0.35 & 2.62 & 4.08 & 28.30 & 14.15 & 5.66 & 2.83 & 0.57      & 2.88 & 3.12 & 4.70 & 82.31 & 41.16 & 16.46 & 8.23 & 1.65 \\         
& $2^{16}$ & 0.004 & 1.32 & 2.24 & 12.38 & 6.19 & 2.48 & 1.24 & 0.25 & 0.31 & 2.74 & 4.30 & 28.82 & 14.41 & 5.76 & 2.88 & 0.58      & 1.23 & 3.40 & 5.14 & 85.41 & 42.70 & 17.08 & 8.54 & 1.71 \\         
\midrule
\multirow[c]{13}{*}{\rotatebox{90}{IWAE-10}} 
& $2^{3}$ & 2.73 & 0.58 & 0.88 & 12.69 & 6.35 & 2.54 & 1.27 & 0.25 & 3.39 & 0.61 & 0.93 & 14.86 & 7.43 & 2.97 & 1.49 & 0.30 & 8.81 & 0.66 & 0.99 & 31.32 & 15.66 & 6.26 & 3.13 & 0.63 \\
& $2^{4}$ & 1.71 & 0.75 & 1.14 & 14.34 & 7.17 & 2.87 & 1.43 & 0.29   & 2.70 & 0.81 & 1.23 & 17.69 & 8.84 & 3.54 & 1.77 & 0.35  & 7.18 & 0.88 & 1.32 & 37.80 & 18.90 & 7.56 & 3.78 & 0.76 \\
& $2^{5}$ & 1.09 & 0.90 & 1.38 & 15.46 & 7.73 & 3.09 & 1.55 & 0.31   & 2.15 & 1.01 & 1.53 & 20.01 & 10.01 & 4.00 & 2.00 & 0.40 & 6.14 & 1.10 & 1.66 & 43.38 & 21.69 & 8.68 & 4.34 & 0.87 \\
& $2^{6}$ & 0.67 & 1.03 & 1.59 & 16.25 & 8.13 & 3.25 & 1.63 & 0.33   & 1.76 & 1.20 & 1.83 & 21.94 & 10.97 & 4.39 & 2.19 & 0.44 & 5.37 & 1.32 & 1.99 & 48.23 & 24.11 & 9.64 & 4.82 & 0.96 \\
& $2^{7}$ & 0.42 & 1.14 & 1.78 & 16.80 & 8.40 & 3.36 & 1.68 & 0.34   & 1.48 & 1.39 & 2.12 & 23.61 & 11.80 & 4.72 & 2.36 & 0.47 & 4.81 & 1.54 & 2.32 & 52.48 & 26.24 & 10.50 & 5.25 & 1.05 \\
& $2^{8}$ & 0.27 & 1.23 & 1.94 & 17.19 & 8.59 & 3.44 & 1.72 & 0.34   & 1.26 & 1.57 & 2.41 & 25.06 & 12.53 & 5.01 & 2.51 & 0.50 & 4.33 & 1.76 & 2.65 & 56.37 & 28.19 & 11.28 & 5.64 & 1.13 \\
& $2^{9}$ & 0.17 & 1.30 & 2.07 & 17.46 & 8.73 & 3.49 & 1.75 & 0.35   & 1.08 & 1.76 & 2.69 & 26.34 & 13.17 & 5.27 & 2.63 & 0.53 & 3.92 & 1.98 & 2.98 & 59.87 & 29.94 & 11.97 & 5.99 & 1.20 \\
& $2^{10}$ & 0.10 & 1.36 & 2.18 & 17.64 & 8.82 & 3.53 & 1.76 & 0.35  & 0.91 & 1.93 & 2.97 & 27.48 & 13.74 & 5.50 & 2.75 & 0.55 & 3.61 & 2.20 & 3.31 & 62.97 & 31.49 & 12.59 & 6.30 & 1.26 \\
& $2^{11}$ & 0.07 & 1.41 & 2.27 & 17.77 & 8.89 & 3.55 & 1.78 & 0.36  & 0.78 & 2.10 & 3.24 & 28.48 & 14.24 & 5.70 & 2.85 & 0.57 & 3.34 & 2.42 & 3.65 & 66.00 & 33.00 & 13.20 & 6.60 & 1.32 \\
& $2^{12}$ & 0.05 & 1.45 & 2.34 & 17.86 & 8.93 & 3.57 & 1.79 & 0.36  & 0.67 & 2.27 & 3.50 & 29.37 & 14.69 & 5.87 & 2.94 & 0.59 & 3.07 & 2.60 & 3.92 & 68.31 & 34.16 & 13.66 & 6.83 & 1.37 \\
& $2^{13}$ & 0.03 & 1.48 & 2.41 & 17.93 & 8.96 & 3.59 & 1.79 & 0.36  & 0.52 & 2.43 & 3.75 & 30.18 & 15.09 & 6.04 & 3.02 & 0.60 & 2.82 & 2.78 & 4.19 & 70.36 & 35.18 & 14.07 & 7.04 & 1.41 \\
& $2^{14}$ & 0.02 & 1.50 & 2.46 & 17.97 & 8.99 & 3.59 & 1.80 & 0.36  & 0.44 & 2.57 & 3.97 & 30.85 & 15.42 & 6.17 & 3.08 & 0.62 & 2.70 & 2.94 & 4.44 & 72.31 & 36.16 & 14.46 & 7.23 & 1.45 \\
& $2^{15}$ & 0.01 & 1.52 & 2.51 & 18.01 & 9.00 & 3.60 & 1.80 & 0.36  & 0.35 & 2.68 & 4.16 & 31.36 & 15.68 & 6.27 & 3.14 & 0.63 & 1.87 & 3.17 & 4.78 & 74.71 & 37.36 & 14.94 & 7.47 & 1.49 \\
& $2^{16}$ & 0.01 & 1.54 & 2.55 & 18.03 & 9.01 & 3.61 & 1.80 & 0.36     & 0.28 & 2.83 & 4.40 & 32.00 & 16.00 & 6.40 & 3.20 & 0.64  &  0.72 & 3.22 & 4.86 & 75.46 & 37.73 & 15.09 & 7.55 & 1.51 \\
\bottomrule
\end{tabular}
\end{table}

\paragraph{Algorithmic Details}

In each experiment, to calculate the required bounds, \mm{i.e.,} CUBO-VG-B, EUBO-VG-B, TVO-VG-B, and (our) IS-VG-B,  we used \eqref{app_cubo}, \eqref{app_cubo-vg-b}, \eqref{app_tvol}, \eqref{app_tvou}, \eqref{app_tvo-vg-b}, and \eqref{app_eubo-vg-b}, as well as \eqref{eq:exp:4}, \eqref{eq:isvgb}, and \eqref{eq:cxopt} from the main paper. It means that we had \mm{to} estimate a few expected values that appear \mm{in the mentioned formulas.} We made this by both \mm{direct application of respective (unbiased) sample mean estimators and a method proposed by \citet{masrani2019thermodynamic} (see (13) therein)} that allowed us to reuse samples from $q(\cdot|x)$ to estimate an expected value under any $\pi_{\beta_l}$. Namely, we used the following formulas:
\begin{equation}\label{app_samplemeanest1}
\E_{z\sim q(\cdot|x)}R(x,z)^n\approx \frac{1}{k}\sum_{i=1}^k R(x,z_i)^n \text{ for } n\geq 1,
\end{equation}
\begin{equation}\label{app_samplemeanest4}
\E_{z\sim \pi_{\beta_l}}\log R(x,z)\approx \sum_{i=1}^k \overline{w_i^{l}}\log R(x,z_i) \text{ for } l=0,\ldots, K,
\end{equation}
\begin{equation}\label{app_samplemeanest2}
\E_{z_i \sim q(\cdot |x)} \log  \frac{1}{k}\sum_{i=1}^k R(x,z_i)\approx \frac{1}{m}\sum_{j=1}^m\log  \frac{1}{k}\sum_{i=1}^k R(x,z^j_i),
\end{equation}
\begin{equation}\label{app_samplemeanest3}
\E_{z_i, \tilde z_i \sim q(\cdot |x)}  \frac{  \sum_{i=1}^k R(x,\tilde z_i)}{ \sum_{i=1}^k R(x,z_i)}\approx \frac{1}{m}\sum_{j=1}^m\frac{\sum_{i=1}^k R(x,\tilde z^j_i)}{ \sum_{i=1}^k R(x,z^j_i)},
\end{equation}
where $z_i$, $z_i^j$, and $\tilde z_i^j$ were sampled from $q(\cdot|x)$,
$w_i^l=R(x,z_i)^{\beta_l}$, and $\overline{w_i^l}=w_i^l/\sum_{j=1}^k w_j^l$. In each \mm{case,} we forced the use of the same number of latent samples ($z_i$ or $z_i^j$), and therefore computations of CUBO and TVO were averaged over $m$ repetitions. In our experiments we set $m=3$ and $k=2^3,\ldots, 2^{16}$.

\paragraph{Training and Architecture Details} All experiments were performed on a single Tesla v100 GPU. We used a convolutional VAE architecture (see below for more details) with weights optimized by the Adam optimizer and a learning rate of 0.0001. The networks were trained for 100 epochs with a batch size of 64 for \mm{the CelebA} dataset and 50 epochs for both \mm{MNIST and SVHN} datasets.
In all reported experiments, we used Euclidean latent spaces $\mathcal{Z}=\R^d$ for $d=8, 32, 128$, depending on the used dataset (respectively: MNIST, SVHN, and CelebA). 
We took standard Gaussian priors $p(z) \sim \mathcal{N}(0, I_d)$.
We used Gaussian encoders $q(z|x) \sim  \mathcal{N}(\mu_x, \Sigma_x)$, with a mean $\mu_x$ and a diagonal covariance matrix $\Sigma_x$, and Gaussian decoders $p(x|z) \sim  \mathcal{N}(q(z|x), \sigma^2 I)$ with $\sigma^2 = 0.3$.

We applied three different models, each for a different dataset. For the MNIST dataset, we used a network architecture that contained two parts: an encoder and a decoder. The encoder consisted mainly of a 2-layer fully-connection, and the decoder consisted of a 3-layer fully-connection. Between each layer, we used the ReLu activation function.

\begin{figure*}[t]
    \centering
    \includegraphics[width=.32\textwidth]{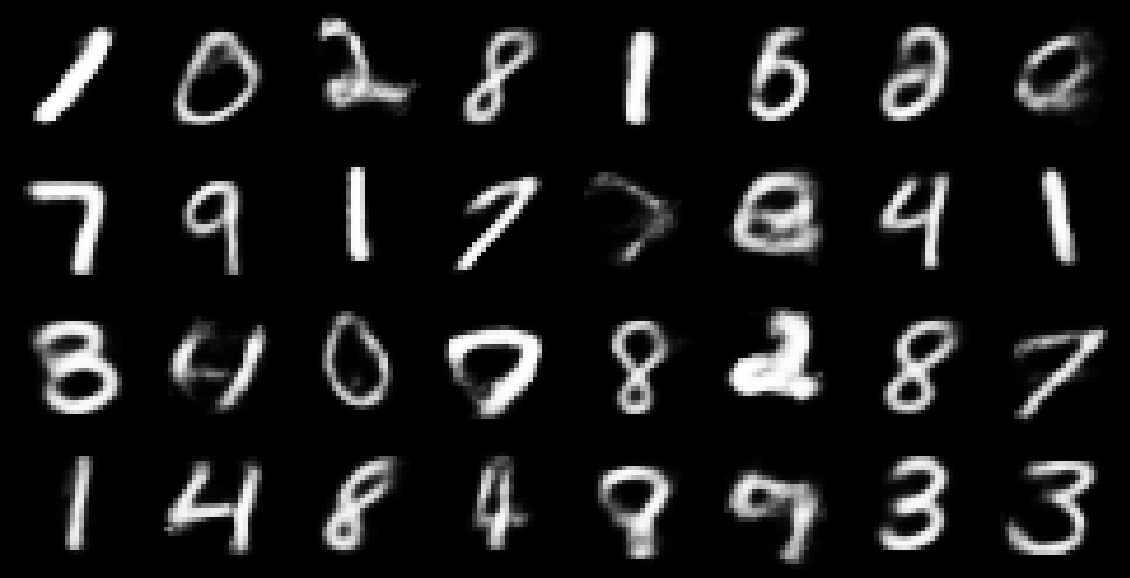}~
    \includegraphics[width=.32\textwidth]{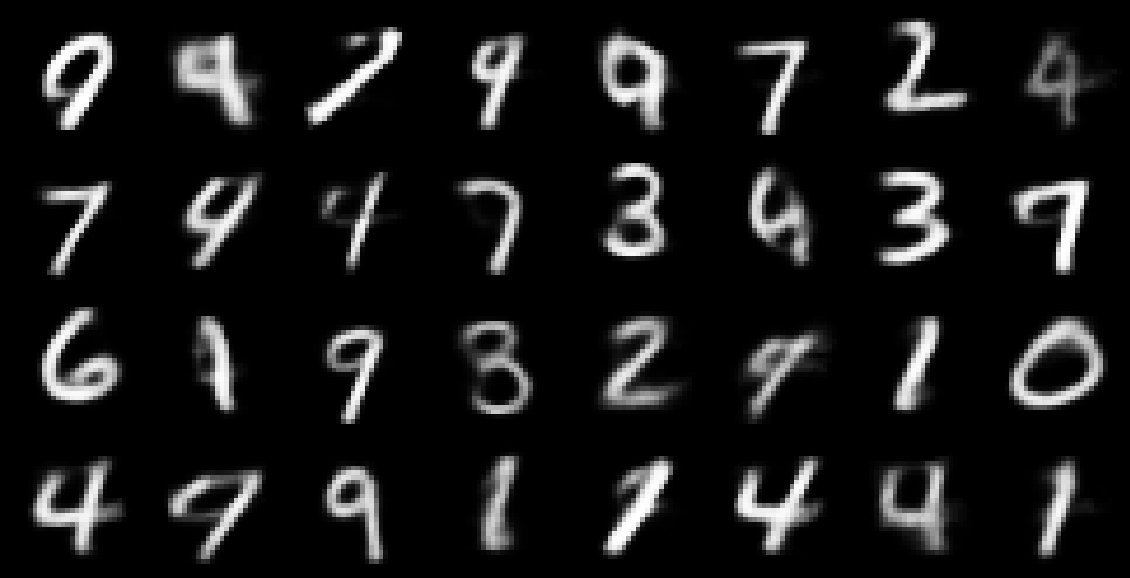}~
    \includegraphics[width=.32\textwidth]{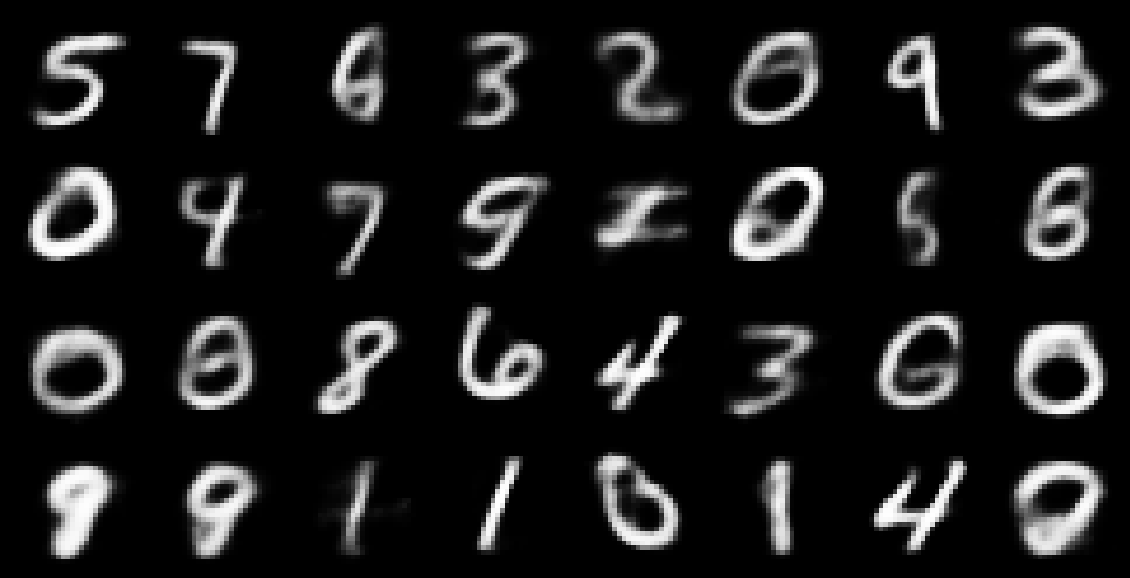} \\[3pt]
    \includegraphics[width=.32\textwidth]{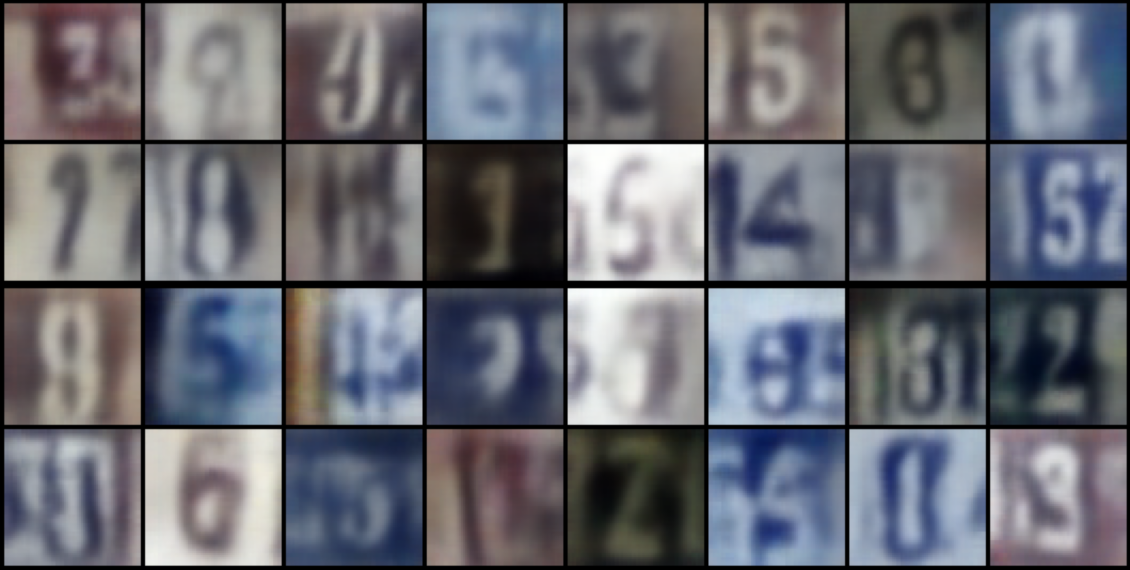}~
    \includegraphics[width=.32\textwidth]{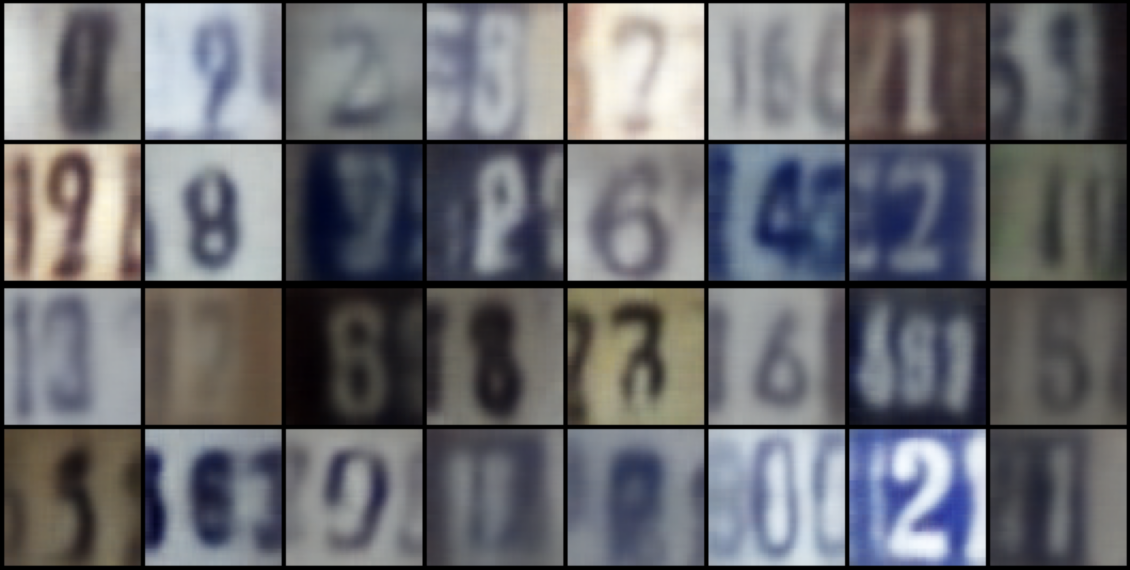}~
    \includegraphics[width=.32\textwidth]{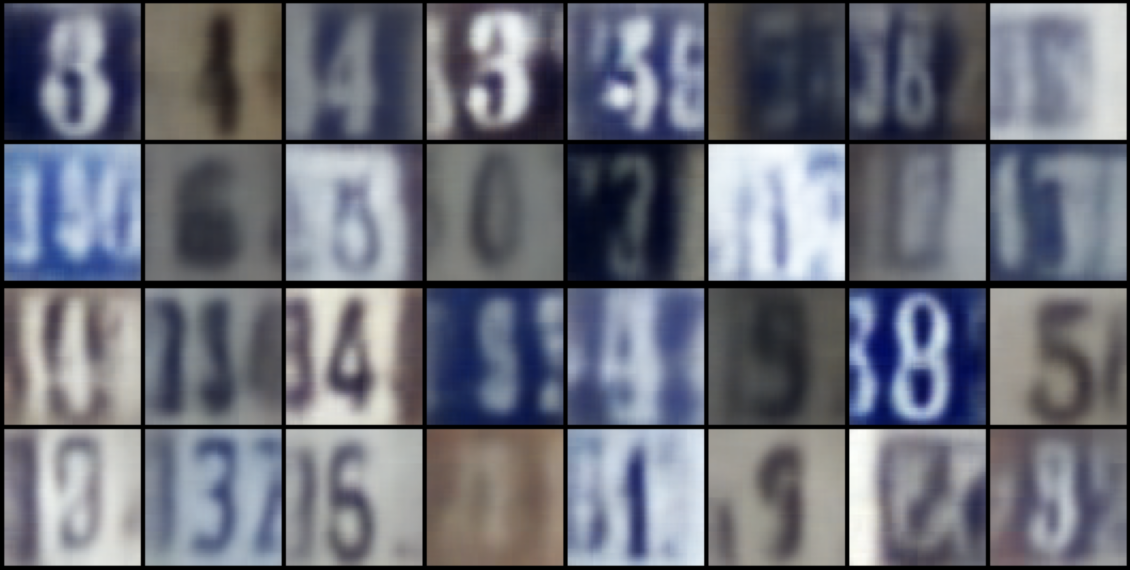} \\[3pt]
    \includegraphics[width=.32\textwidth]{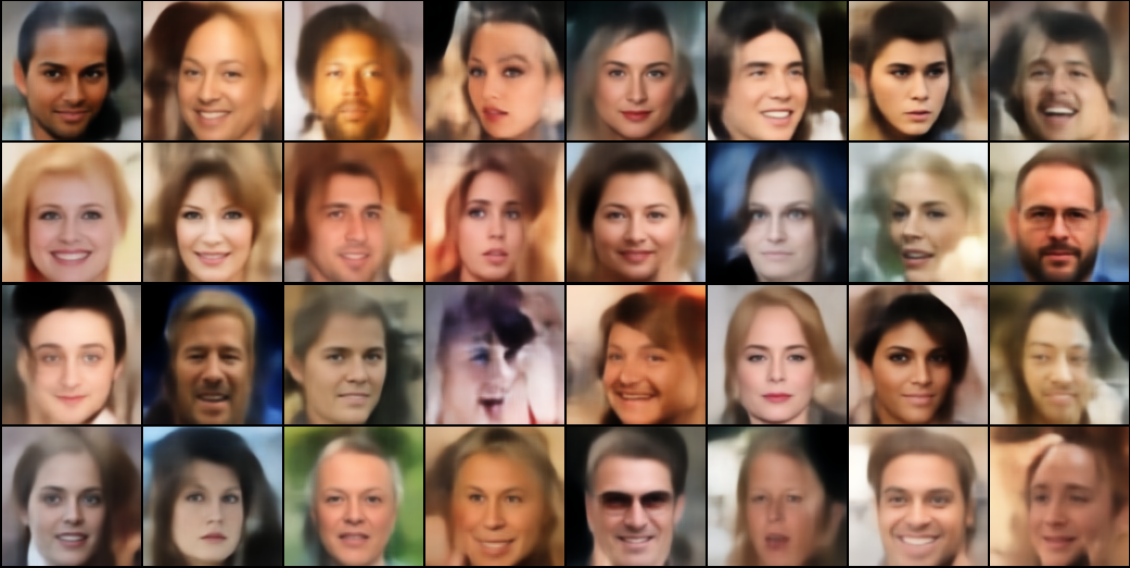}~
    \includegraphics[width=.32\textwidth]{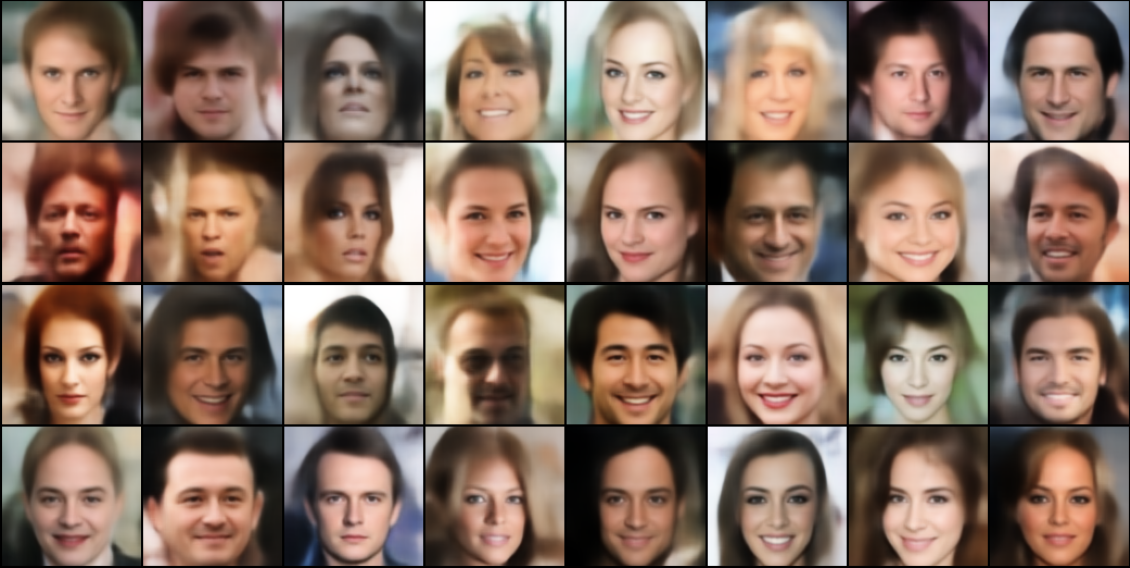}~
    \includegraphics[width=.32\textwidth]{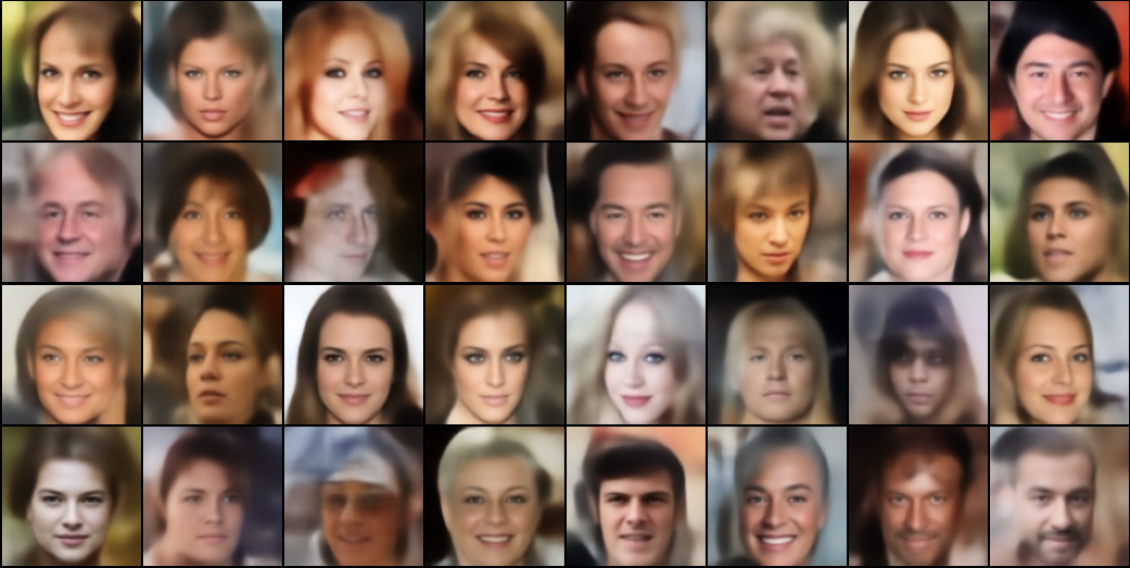}
    
    \caption{Sampled images for the VAE, \mm{IWAE-5, and} IWAE-10 models (from left to right) trained on MNIST, \mm{SVHN, and} CelebA datasets (from top to bottom).}
    \label{fig.samples_images}
\end{figure*}

In the case of the SVHN dataset, we used deeper architectures for \mm{the encoder} and decoder. Both networks consisted mainly of 4 layers. The encoder had only 2D convolutions, between which we used leaky ReLU with leakiness \mm{of 0.2.} In the decoder, we applied \mm{2D} transposed convolutions and ReLu as activation functions. For the last layer, we used the sigmoid activation function. 

For the CelebA dataset, we applied network architectures that consisted mainly of replicated 5-layer blocks. In the encoder network, each block was built with a 2D convolution layer, batch normalization~\citep{ioffe2015batch}, and leaky ReLU with leakiness \mm{of} 0.2.  A single block in the decoder network contained \mm{an operation} that applied a 2D \mm{nearest-neighbor} upsampling to an input signal composed of several input channels. Then, similar to the block of the encoder network, there was a 2D convolution layer, batch normalization, and leaky ReLU with \mm{a leakiness of 0.2.}

\paragraph{Details for Experiment with Laplace Distributed Data}
\mm{We generated $10^4$ observations from Laplace distribution $\textrm{Laplace}(0, 0.2)$. Then we use them to calculate the (average) log-likelihood, ELBO, IW-ELBO lower bound, and our proposed upper bound (taking $C = 0$) for the previously learned VAE with one-dimensional latent space. To estimate lower and upper bounds, we sampled from the latent different numbers of times (from 1 to 64), to examine how the number of draws will affect bound positions. In the training procedure of the underlying VAE model,} we selected batches of size 1000 and SGD optimizer with the learning rate $10^{-7}$. The encoder consisted of one hidden layer with 4 neurons and \mm{a ReLU} activation function in it, with linear activation in the output layer. We set an identical network architecture for the decoder.

\mm{\paragraph{Simple Qualitative Evaluation}
We tried to confirm the quantitative results presented in \Cref{fig.compare_models} in the main paper by comparing visually samples randomly generated by various models. We expected to obtain images of the best quality for VAE trained on MNIST and for IWAE models trained on SVHN, as well as images of comparable quality for all considered models trained on CelebA. The respective samples are presented in ~\Cref{fig.samples_images}. Although some slight visual effects might be visible after a closer look, the differences are not impressive. Hence, in this case, simple qualitative evaluation is a rather inadequate method. }

\vfill

\end{document}